\documentclass{article}

\usepackage[
    top=1in,
    bottom=1in,
    left=1.25in,
    right=1.25in
]{geometry}

\usepackage{amsmath}
\usepackage{amssymb}

\usepackage{hyperref}
\usepackage{cleveref}
\usepackage{lipsum}
\usepackage{amsfonts}
\usepackage{graphicx}
\usepackage{epstopdf}
\usepackage{algorithmic}
\usepackage[dvipsnames]{xcolor}

\usepackage{url}

\usepackage{booktabs}
\usepackage{subcaption}

\newcommand{\proofbox}{\qedsymbol}
\usepackage{amsthm}

\newtheorem{Thm}{Theorem}[section] 

\newtheorem{Prop}[Thm]{Proposition} 

\newtheorem{Lem}[Thm]{Lemma}
\newtheorem{Assum}[Thm]{Hypothesis}

\theoremstyle{definition}
\newtheorem{Def}[Thm]{Definition}

\newtheorem{Exp}[Thm]{Example}

\theoremstyle{remark}
\newtheorem{Rem}[Thm]{Remark}

\crefname{Thm}{Theorem}{Theorems}
\Crefname{Thm}{Theorem}{Theorems}
\crefname{Lem}{Lemma}{Lemmas}
\Crefname{Lem}{Lemma}{Lemmas}
\crefname{Prop}{Proposition}{Propositions}
\Crefname{Prop}{Proposition}{Propositions}

\crefname{Assum}{Hypothesis}{Hypotheses}
\Crefname{Assum}{Hypothesis}{Hypotheses}
\crefname{Prob}{Problem}{Problems}
\Crefname{Prob}{Problem}{Problems}
\Crefname{Prop}{Proposition}{Propositions}

\title{Hypothesis-driven construction of mesoscopic dynamics
\thanks{
This project is supported by the National Research Foundation, Singapore, under its AI Singapore Programme (AISG Award No.: AISG3-RP-2022-028).
Z.L. is supported by the Ministry of Education, Singapore, under its Research Centre of Excellence award to the Institute for Functional Intelligent Materials. (Project No. EDUNC-33-18-279-V12).
}
}
\newcommand{\email}[1]{\texttt{#1}}

\author{
Zhuoyuan Li\footnotemark[2]\;\;\footnotemark[4]
\and Aiqing Zhu\footnotemark[3]\;\;\footnotemark[4]
\and Qianxiao Li\footnotemark[2]\;\;\footnotemark[3]\;\;\footnotemark[5]
}
\date{}
\usepackage{amsopn}

\newcommand{\mR}{\mathbb{R}}

\newcommand{\mC}{\mathbb{C}}
\newcommand{\mZ}{\mathbb{Z}}
\newcommand{\mT}{\mathbb{T}}
\newcommand{\md}{\mathop{}\!\mathrm{d}}
\newcommand{\mE}{\mathcal{E}}

\newcommand{\mG}{\mathcal{G}}
\newcommand{\mX}{\mathcal{X}}

\newcommand{\mH}{\mathcal{H}}
\newcommand{\mL}{\mathcal{L}}

\newcommand{\mP}{\mathcal{P}}
\newcommand{\mM}{\mathcal{M}}
\newcommand{\mW}{\mathcal{W}}

\newcommand{\mJ}{\mathcal{J}}
\newcommand{\norm}[1]{\left\| #1 \right\|}
\newcommand{\abs}[1]{\left| #1 \right|}

\newcommand{\me}{\mathrm{e}}
\newcommand{\mi}{\mathrm{i}}

\newcommand{\ddt}[1]{\frac{\md{#1}}{\md t}}

\newcommand{\init}{\mathrm{init}}

\newcommand{\Hinner}[2]{\left( #1,#2\right)}
\newcommand{\Xcouple}[2]{\left\langle #1,#2\right\rangle_{\mX^*,\mX}}

\newcommand{\ournet}{{Spectral OnsagerNet}}
\newcommand{\ournetabbrv}{{SpecOnsNet}}

\begin{document}
\maketitle

% 3. Footnote definitions
\renewcommand{\thefootnote}{\fnsymbol{footnote}}
\footnotetext[2]{Institute for Functional Intelligent Materials, National University of Singapore (\email{zy.li@nus.edu.sg}, \email{qianxiao@nus.edu.sg})}
\footnotetext[3]{Department of Mathematics, National University of Singapore (\email{zaq@nus.edu.sg}).}
\footnotetext[4]{Zhuoyuan Li and Aiqing Zhu contributed equally to this work.}
\footnotetext[5]{Corresponding author.}
\renewcommand{\thefootnote}{\arabic{footnote}}

% REQUIRED
\begin{abstract}
% \todo
Traditional scientific modeling typically begins with fixed, instance-wise effective equations and then carries out equation-specific analysis and computation, a procedure that becomes exceptionally challenging in complex applications such as multiscale systems.
We propose an alternative paradigm by learning mesoscopic dynamics within a mathematically constrained hypothesis class.
Building upon a generalized Onsager principle, we introduce a unified framework encompassing both dissipative and conservative mesoscopic dynamics.
We establish uniform and \textit{a priori} theoretical guarantees, including global well-posedness, asymptotic stability, unique factorization identifiability, and discrete energy dissipation, applicable to all spatio-temporal evolution equations within this hypothesis class prior to all learning stages.
Data from each problem instance is then used to guide the identification of members within our hypothesis class,
giving rise to accurate, robust and interpretable dynamical models.
We empirically validate this framework on both data from continuum PDE models as a check,
and on data arising from microscopic chain models for which exact meso-scale models are unknown.
The proposed approach not only acts as an effective dynamics learner, but also offers vital interpretable diagnostics of the underlying physics.
% This work represents a concrete step towards a hypothesis-driven paradigm for the discovery of reliable, physically consistent mesoscale models.
% ~155 words
% The Onsager principle is a cornerstone of non-equilibrium thermodynamics that governs energy dissipation and transport across physical, chemical, and biological systems, yet its data-driven realization has so far been confined to finite-dimensional settings. Generalizing it to infinite dimensions is essential for modeling spatially extended dynamics but raises fundamental challenges in operator parameterization, structural enforcement, and function-space analysis.  We propose a generalized Onsager principle for infinite-dimensional systems that admits a clean spectral representation, and establish global well-posedness, asymptotic stability, identifiability of the learned components, and discrete energy dissipation.  A neural network realization, the Spectral OnsagerNet, enforces these guarantees by construction and is validated on representative dissipative and conservative systems.  Applied to microscopic particle-chain models, the framework discovers interpretable mesoscopic governing equations directly from microscopic data, opening a principled route from microscopic observations to macroscopic thermodynamic descriptions.
\end{abstract}

\providecommand{\keywords}[1]{
    \vspace{0.2cm} % Adds a little space below the abstract
    \noindent \small \textbf{\textit{Keywords:}} #1
}
% REQUIRED

% REQUIRED
% \begin{MSCcodes}
% 00A71, % General theory of mathematical modeling
% 35M11, % Initial value problems for PDEs of mixed type
% 37M10, % Time series analysis of dynamical systems
% 82C05, % Classical dynamic and nonequilibrium statistical mechanics (general)
% 82C26 % Dynamic and nonequilibrium phase transitions (general) in statistical mechanics
% \end{MSCcodes}

\section{Introduction}
Mathematical modeling of dynamical processes is traditionally organized around a well-established pipeline \cite{weinan2011principles,oberkampf2010verification,pavliotis2008multiscale}:
one starts from a physical law in the form of a mathematical equation with a finite set of parameters, uses experiments or observations to determine these parameters, thereby obtains a fixed governing equation;
One then studies the well-posedness of the resulting model, designs stable and accurate discretizations,
and analyzes the resulting simulations.
This framework has been highly successful,
and scientific modeling and computing have accordingly centered on two tasks: the identification of an effective governing equation \cite{weinan2011principles,kennedy2001bayesian} and, once such an equation is available, its instance-wise mathematical and numerical analysis \cite{brenner2008mathematical,evans2022partial}.
% but its analytical organization is largely instance-wise: the theory is developed for a particular model once the effective equation is available \cite{brenner2008mathematical,evans2022partial}. Moreover, the corresponding analysis certifies the chosen equation, but not necessarily its fidelity as a representation of the underlying system, since such equations are often obtained only after approximation, closure, averaging, or phenomenological simplification \cite{weinan2011principles,kennedy2001bayesian,kevrekidis2009equation}.

In many problems of interest, the effective equation itself emerges from underlying microscopic descriptions through various limiting procedures, such as kinetic or hydrodynamic limits, homogenization or averaging arguments \cite{weinan2011principles,kevrekidis2009equation}. However, it is often unclear how to derive the limiting equation in a systematic way, and even when such a derivation is available, favorable analytical properties of the approximating models need not pass automatically to the limit.
This phenomenon is typical in multiscale dynamics, where the microscopic dynamics are often directly specified and analytically more accessible (yet expensive to simulate), whereas the corresponding mesoscopic spatio-temporal evolution equations are much more subtle to derive and analyze. Representative examples include the passage from particle chains to KdV-type equations \cite{bambusi2006metastability,zabusky1965interaction}, from molecular dynamics to kinetic equations \cite{weinan2011principles,gallagher2013newton}, and from the asymmetric exclusion process to the KPZ equation \cite{bertini1997stochastic,corwin2012kardar}.

These considerations motivate us to leverage structured approaches and deep learning to circumvent these difficulties. Under suitable hypotheses, one can construct a class whose elements are admissible equations involving operators, functionals, or functions as unknown components. The class is chosen to balance two requirements: it must be sufficiently expressive for some elements of it to represent the underlying system with adequate fidelity, and sufficiently structured for well-posedness, stability, and related analytical properties to be studied uniformly across the class. Training is then used to identify a specific representative within such a class. In this way, one works within a subspace of possible equations that is both close enough to the target system for modeling purposes and sufficiently structured to support a family-level analysis. This change of framework is summarized in \Cref{fig:intro}.

\begin{figure}
    \centering
    \includegraphics[width=\linewidth]{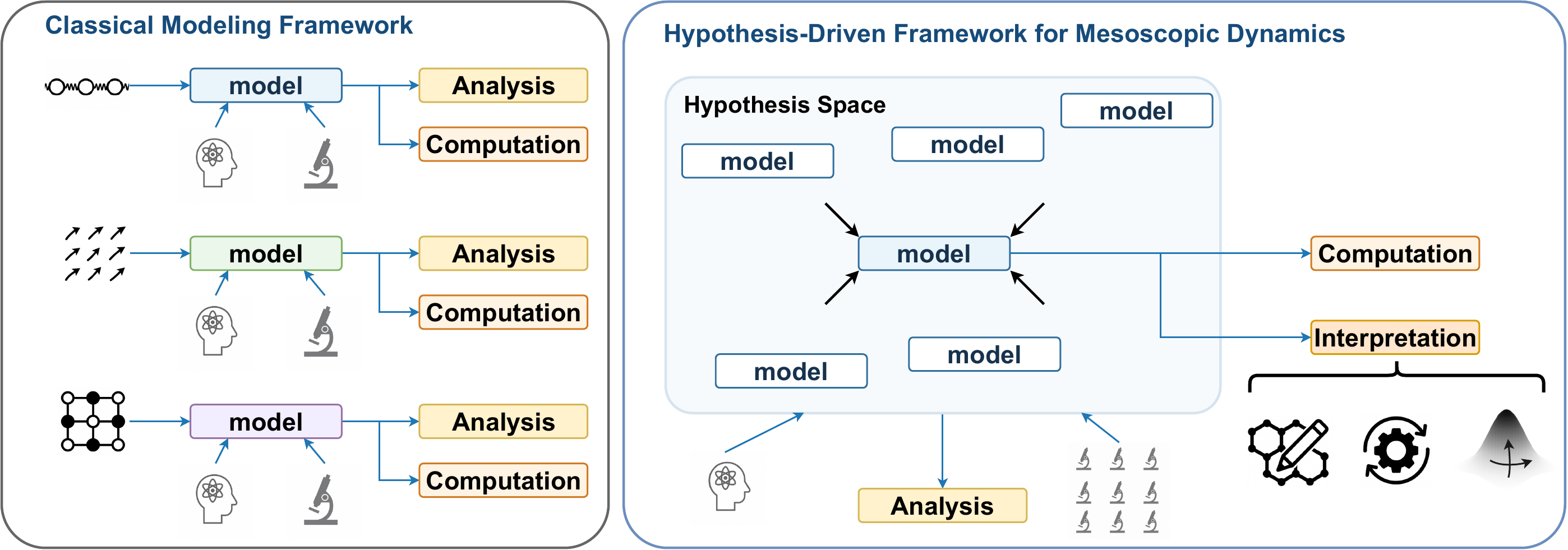} 
    \caption{Comparison of our framework and classical modeling framework. Classical approaches treat modeling and analysis case by case. Our approach instead first builds a hypothesis space with a unified analytical theory and then uses data to identify an admissible model from this space.}
    \vskip -0.5cm
    \label{fig:intro}
\end{figure}
In this paper, we demonstrate this program for the specific class of systems admitting a possibly infinite-dimensional Onsager-type dynamics, formulated as
\begin{equation}\label{eq:onsager-factorization-intro}
    \partial_t u = -\left[\mM(u) + \mW(u)\right]\frac{\delta V}{\delta u},
\end{equation}
which will be explained later.
In addition, we choose some structural hypotheses under which the operator constraints admit a pointwise spectral characterization in terms of Fourier multipliers. This defines a hypothesis class broad enough to include both dissipative and conservative thermodynamic systems, yet structured enough to admit a unified analysis.
Within this setting, we prove global well-posedness and asymptotic stability. By securing these properties, we establish an a priori analytical foundation for the entire hypothesis class, ensuring that our models are endowed with physical stability independent of, and prior to, any data-driven learning process. We then establish identifiability of the Onsager factorization \cref{eq:onsager-factorization-intro} up to a one-dimensional rescaling, and derive explicit conditions under which the forward Euler scheme preserves the energy-dissipation inequality. For numerical validation, we introduce \ournet, a neural architecture that enforces the structural constraints by construction. Numerical experiments on data generated from the KdV and Allen--Cahn equations confirm accuracy, stability, and recovery of the expected physical structure. 
Experiments on microscopic particle chains, including an FPUT chain \cite{fermi1955studies} with a known KdV-type continuum limit \cite{zabusky1965interaction} and an FENE chain \cite{warner1972kinetic} without a known closed-form continuum limit, further show that the method can recover interpretable mesoscopic dynamics directly from microscopic trajectory data.

The remainder of this paper is organized as follows. After discussing our approach in the context of the existing literature in \cref{sec:related}, \cref{sec:onsager} formally introduces the generalized Onsager principle on a Gelfand triple and establishes its spectral representation. \Cref{sec:theoretical-analysis} then provides the theoretical foundation for our framework, including rigorous analytical guarantees for global well-posedness, asymptotic stability, factorization identifiability, and discrete energy dissipation. Transitioning from theory to computational methodology, \cref{sec:network} details the neural architecture of \ournet. \Cref{sec:experiments} validates the proposed method through numerical experiments on both continuum PDE models and microscopic particle chains, and \cref{sec:conclusion} concludes the paper. For clarity of exposition, we only provide proof sketches for some of the propositions, and detailed versions are deferred to \cref{sm:proofs}.

\section{Related works}\label{sec:related}
The identification of governing equations from observational data has been studied extensively \cite{Brunton2016SINDy,chen2018neural,Champion2019,brunton2019data,du2022discovery}.  
Most existing works are formulated for finite-dimensional systems or for finite-dimensional state representations inferred from data.  
However, many problems of interest in scientific computing, including fluid flows, phase transitions, and pattern formation, are governed by evolution equations on infinite-dimensional state spaces.  
Accordingly, the object to be learned is not merely a vector field on $\mathbb R^d$, but an evolution system posed on a function space.

Recent data-driven approaches address such problems in two main ways.  
One line of work models the dynamics through finite-dimensional state-space or latent representations, as in ODE-Net \cite{chen2018neural,chen2025due}.  
Another line learns the solution map or evolution operator more directly, as in neural operator methods such as the Fourier Neural Operator \cite{li2020fourier} and DeepONet \cite{lu2021learning}.  
While these approaches provide expressive approximations of nonlinear dynamics, they typically treat the learned vector field or operator as an unconstrained map.  
As a result, they generally do not identify structural features of the dynamics, such as an underlying energy, a dissipative--conservative splitting, or a built-in stability mechanism, which are often central to the analysis of infinite-dimensional evolution systems.

A natural strategy to impose such structure is through physics-informed model classes rooted in variational or thermodynamic principles. Representative examples include Hamiltonian models for conservative systems \cite{bertalan2019learning,celledoni2023learning,greydanus2019hamiltonian,hansen2025learning, wu2020structure} and Onsager- or GENERIC-type architectures for dissipative dynamics \cite{Onsager1931a,Onsager1931b,Doi2011,Doi2015,Yu2021OnsagerNet,lee2021machine,zhang2022gfinns}. While related concepts have been successfully adapted for stochastic systems and model reduction \cite{Chen2023NCSConstructingCustomThermodynamics,zhu2025identifiable}, these frameworks remain predominantly confined to finite-dimensional settings. A notable recent exception is Stat-PINNs \cite{huang2025statistical}, which incorporates thermodynamic constraints into infinite-dimensional evolution systems; however, it is strictly limited to equilibrium systems. Consequently, there remains a critical need for structure-preserving frameworks capable of learning non-equilibrium dynamics in infinite-dimensional spaces.

\section{Generalized Onsager principle}\label{sec:onsager}
The objective of this paper is to model the evolution of physical fields of our interest over a spatial domain, specifically targeting mesoscopic dynamics. We characterize this mesoscopic regime as the continuous spatio-temporal evolution that bridges the critical gap between discrete, microscopic particle interactions and bulk, macroscopic thermodynamics. By capturing these intricate spatial and temporal dependencies, we seek to develop a mesoscopic model that is globally well-posed, thermodynamically consistent, and stable for long-time evolution. This section establishes the continuous mathematical framework necessary to achieve this goal. We first extend the classical finite-dimensional Onsager principle to infinite-dimensional spaces. After that, we introduce the core structural hypotheses, including the rigorous formulation of the underlying function spaces, that render the hypothesis class analytically and computationally tractable. Finally, we derive a spectral representation of these dynamics, bridging the continuous thermodynamic theory with the practical neural network implementation.

\subsection{From finite to infinite dimensions}
\label{sec:finite-to-infinite}

In finite dimensions, learning from trajectories is often formulated as learning an ODE vector field $\dot{z}_t = f(z_t)$.
The Neural ODE approach parameterizes $f$ as an unconstrained neural network \cite{chen2018neural}, which offers expressive power but does not generally preserve the structure underlying physical systems.
OnsagerNet provides a structured alternative for finite-dimensional dissipative dynamics \cite{Yu2021OnsagerNet,Chen2023NCSConstructingCustomThermodynamics}.
Rather than treating the vector field as an unconstrained black box, it decomposes the dynamics into a dissipative part and a conservative part, thereby extending the classical Onsager principle \cite{Onsager1931a} to strongly non-equilibrium regimes.
Specifically, the evolution for the time-dependent state variable $z_t$ is modeled as
\begin{equation}\label{eq:onsagernet-ode}
    \dot{z}_t = -[M(z_t) + W(z_t)] \nabla V(z_t),
\end{equation}
where $M(z_t)$ is a symmetric and positive semi-definite matrix, $W(z_t)$ is an anti-symmetric matrix, and $V(z_t)$ is a scalar function modeling a potential energy.
Each component is parameterized by a neural network and inferred from data.
By embedding the dissipative--conservative decomposition into the parameterization, OnsagerNet enforces thermodynamic consistency and improves long-term stability, which is generally not guaranteed by unconstrained Neural ODE approaches. These features make it suitable for learning finite-dimensional dissipative dynamics, with applications including Rayleigh--Bénard convection and polymer-chain stretching dynamics \cite{chen2018neural, Yu2021OnsagerNet}.

We next extend this principle to mesoscopic dynamics, where the evolving state is a spatial field rather than a finite-dimensional vector. Typical examples include strain fields in nonlinear particle chains, density fields in interacting particle systems and phase fields in phase-separation models. 
Such field-level descriptions provide spatially resolved effective closures of microscopic dynamics, thereby revealing how local microscopic laws generate wave propagation, energy transfer, pattern formation, phase separation, and relaxation at larger scales.

Consider a free-energy functional $V : \mX \to \mR$ and its Fr\'{e}chet derivative $\mu(u) = \delta V / \delta u \in \mX^*$, which plays the role of the generalized thermodynamic driving force \cite{NonequilibriumDynamics2010}.
We define the linear operators $\mM(u)$ and $\mW(u)$ mapping from the dual space (forces) to the primal space (fluxes/rates):
\begin{equation}
\mM(u), \mW(u) : \mX^* \to \mX.
\end{equation}
The \emph{generalized Onsager principle in infinite dimensions} then takes the form
\begin{equation}\label{eq:onsager-pde}
    \partial_t u = -\left[\mM(u) + \mW(u)\right]\mu(u), \qquad \mu(u) = \frac{\delta V}{\delta u} \in \mX^*.
\end{equation}
However, unlike the finite-dimensional setting, the infinite-dimensional formulation involving state-dependent functionals and operators poses significant computational and theoretical challenges.
A direct implementation via spatial discretization leads to prohibitively high-dimensional models, making naive parameterization and training costly.
Enforcing symmetry or anti-symmetry and positive semi-definiteness adds additional constraints that further increase the computational burden.
From a theoretical perspective, it is also nontrivial to ensure that the learned model is globally well-posed and remains stable over long-time evolution.
These challenges motivate the structural hypotheses we introduce next, which preserve the Onsager structure at the operator level, admit efficient spectral representations, and provide principled guarantees of energy dissipation and long-time stability.

\subsection{Structural hypotheses}
\label{sec:assumptions}
We aim to study a broad dynamic class of the following form
\begin{equation}\label{eq:cauchy}
    \partial_t u(t) = \mG(u(t)), \qquad u(0) = u_\init \in \mX
\end{equation}
with the solution space
\begin{equation}
    W(0, T; \mX) := \bigl\{u \in L^2(0, T; \mX) \bigm| \partial_t u \in L^2(0, T; \mX)\bigr\}.
\end{equation}
In this subsection, we formulate the structural hypotheses used in the subsequent analysis, which keep the hypothesis class analytically tractable while preserving sufficient approximation flexibility.
We start by stating several mild hypotheses on the function space $\mX$, where the field of the physical mesoscopic state of our interest stays.
\begin{Assum}\label[Assum]{ass:X-constraints}
    $\mX$ is a uniformly smooth and uniformly convex Banach space compactly embedded within $\mH=L^2(\Omega)$, and the solution domain $\Omega$ is bounded in $\mR^d$. 
\end{Assum}
Throughout this work, we use $\Hinner{\cdot}{\cdot}$ to denote the inner product on the Hilbert space $\mH=L^2(\Omega)$ and $\Xcouple{\cdot}{\cdot}$ to denote the duality pairing on $\mX$.
 For a Banach space $\mX$ compactly embedded within $\mH=L^2(\Omega)$, it is clear that we have the Gelfand triple
 \begin{equation}
    \mX \hookrightarrow \mH \cong \mH^* \hookrightarrow \mX^*.
\end{equation}
% Let $J:\mX\to\mX^*$ denote the duality mapping, then the uniform smoothness of $\mX$ guarantees that $J$ is single-valued and uniformly continuous on bounded sets. The mapping $J$ can also be defined as the Fr\'echet derivative of $\frac12\|\cdot\|_\mX^2$. 
Moreover, uniformly convexity implies the reflexivity of $\mX$ by Milman--Pettis theorem.

Sobolev spaces offer a concrete class of Banach spaces satisfying \cref{ass:X-constraints}. A Sobolev space $W^{s,p}(\Omega)$ is uniformly smooth and uniformly convex for $1<p<+\infty$, so it suffices to ensure that $W^{s,p}(\Omega)$ is compactly embedded within $L^2(\Omega)$, which is a direct result of the Sobolev embedding theorem. In particular, we need $s\ge d(1/p-1/2)_+$ for the Sobolev space $W^{s,p}(\Omega)$ to satisfy our hypothesis.

Having established the functional spaces, we now introduce the fundamental structural hypotheses governing the dynamics. These conditions are specifically designed to enforce thermodynamic consistency directly at the operator level.
\begin{Assum}[Dissipation structure]\label[Assum]{ass:symmetry}
The nonlinear operator $\mG : \mX \to \mX$ admits a decomposition
\begin{equation}
    \mG(u) = -\mL(u)\mu(u) \quad\text{and}\quad \mL(u) = \mM(u) + \mW(u).
\end{equation}
Here, for each $u\in\mX$, $\mu(u)\in\mX^*$, and $\mL(u)$ is a linear operator from $\mX^*$ to $\mX$. The operators $\mM$, $\mW$ and $\mu$ satisfy
    \begin{enumerate}
        \item $\mM(u):\mX^*\to\mX$ is symmetric and positive-semidefinite:
        % \begin{itemize}
        %     \item  $\Xcouple{\xi}{\mM(u)\eta} = \Xcouple{\eta}{\mM(u)\xi}$, $\forall\xi,\eta\in\mX^*$;
        %     \item $\Xcouple{\xi}{\mM(u)\xi} \ge 0$, $\forall\xi\in\mX^*$.
        % \end{itemize}
        \begin{equation}
            \Xcouple{\xi}{\mM(u)\eta} = \Xcouple{\eta}{\mM(u)\xi}\quad\forall\xi,\eta\in\mX^*
        \end{equation}
        and
        \begin{equation}
            \Xcouple{\xi}{\mM(u)\xi} \ge 0\quad\forall\xi\in\mX^*.
        \end{equation}
        % \begin{equation}
        %     \Xcouple{\xi}{\mM(u)\eta} = \Xcouple{\eta}{\mM(u)\xi} \quad \text{and} \quad \Xcouple{\xi}{\mM(u)\xi} \ge 0\quad\forall\xi,\eta\in\mX^*.
        % \end{equation}
        \item $\mW(u):\mX^*\to\mX$ is skew-symmetric:
        \begin{equation}
            \overline{\Xcouple{\xi}{\mW(u)\eta}} = -\Xcouple{\eta}{\mW(u)\xi}.
        \end{equation}
        \item $\mu(u)\in\mX^*$ is the Fr\'echet derivative of a coercive functional: there exists a Fr\'echet differentiable functional $V:\mX\to\mR$ such that
        \begin{itemize}
            \item $V(u_n)<+\infty$ implies $\norm{u_n}_\mX<+\infty$ for any sequence $\{u_n\}_n\subseteq\mX$;
            \item $\mu(u)=\delta V/\delta u$ as the Fr\'echet derivative of $V$.
        \end{itemize}
    \end{enumerate}
\end{Assum}
These structural hypotheses are a direct infinite-dimensional generalization of the finite-dimensional Onsager principle \cref{eq:onsagernet-ode}. Physically, the functional $V$ acts as a generalized free energy, while the operators $\mM$ and $\mW$ capture the dissipative and conservative transport mechanisms, respectively. A vital consequence of this imposed structure is the continuous dissipation of free energy, which guarantees thermodynamically consistent macroscopic behavior.
\begin{Prop}[Energy dissipation]\label[Prop]{prop:energy-dissipation}
    With \cref{ass:symmetry}, the free energy does not increase along any solution trajectory as
    \begin{equation}\label{eq:energy-dissipation}
        \ddt{} V(u) = \Xcouple{\mu(u)}{\partial_t u} = -\Xcouple{\mu(u)}{\mM(u)\mu(u)} \leq 0.
    \end{equation}
    The skew-symmetric part $\mW$ does not contribute to the energy evolution. Hence, any solution trajectory, if it exists, is uniformly bounded in $\mX$ due to the coercivity of $V$.\proofbox
\end{Prop}

While \cref{prop:energy-dissipation} guarantees that trajectories remain uniformly bounded due to the coercivity of $V$, establishing the global well-posedness of the dynamics class \cref{eq:cauchy} requires two additional hypotheses. The first assumption is physically motivated by translational invariance on periodic domains, though formulated here as a general property of linear operators. Crucially, this convolution structure not only makes the continuous analysis tractable but also admits a diagonalized representation in the frequency domain. This explicitly bridges the continuous theory with our computational method, directly inspiring the spectral training and evolution of \ournet. Meanwhile, following standard well-posedness arguments, we impose local Lipschitz continuity to ensure the existence and uniqueness of the solution flow.

\begin{Assum}[Convolution structure]\label[Assum]{ass:convolution}
    The operators $\mM$ and $\mW$ act as convolutions in the physical domain $\Omega=\mT^d$.
    Namely, there exist $K_\mM(\cdot;u),K_\mW(\cdot;u)\in L^2(\Omega)$ such that for any $\eta\in\mX^*$,
    \begin{equation}
        \mM(u)(\eta)(x)=\Xcouple{\eta}{K_\mM(x-\cdot;u)},\ \mW(u)(\eta)(x)=\Xcouple{\eta}{K_\mW(x-\cdot;u)}.
    \end{equation}
\end{Assum}

% \begin{Assum}[Lipschitz continuity]\label[Assum]{ass:lipschitz}
%     The composite operator $\mL := \mM + \mW$ and the modified chemical potential $\mu_1 := \mu + \beta \Delta$ are Lipschitz continuous on bounded sets with respect to the $L^2$-norm, for some scalar $\beta > 0$.
%     That is, for any $R > 0$, there exist constants $L_1(R), L_2(R) > 0$ such that
%     \begin{equation}\label{eq:lipschitz}
%         \begin{aligned}
%         \|\mu_1(v_1)-\mu_1(v_2)\|&\le L_1(R)\|v_1-v_2\|_2,\\
%         \|\mL(v_1)-\mL(v_2)\|&\le L_2(R)\|v_1-v_2\|_2,
%         \end{aligned}
%         \qquad \forall\, v_1, v_2 \in B_{\mX}(R),
%     \end{equation}
%     where $B_\mX(R):=\{v\in\mX\mid\|v\|_\mX\le R\}$.
% \end{Assum}

\begin{Assum}\label[Assum]{ass:lipschitz-v2}
    The operators $\mL$ and $\mu$ are $\mX$-Lipschitz continuous and the operator $\mG$ is $L^2$-Lipschitz continuous on bounded sets. Formally, there exist $L_1(R)$, $L_2(R)$, and $L_3(R)$ such that
    \begin{equation}
        \begin{aligned}
            \|\mu(v_1)-\mu(v_2)\|_{\mX^*}&\le L_1(R)\|v_1-v_2\|_\mX,\\
            \|\mL(v_1)-\mL(v_2)\|_{\mX^*\to\mX}&\le L_2(R)\|v_1-v_2\|_\mX,\\
            \|\mG(v_1)-\mG(v_2)\|_2&\le L_3(R)\|v_1-v_2\|_2,
        \end{aligned}
        \qquad \forall\, v_1, v_2 \in B_{\mX}(R),
    \end{equation}
    where $B_\mX(R):=\{v\in\mX\mid\|v\|_\mX\le R\}$.
\end{Assum}
% Under these four assumptions, we focus on the following Cauchy problem.

% \begin{Prob}\label[Prob]{prob:cauchy}
%     Consider the Gelfand triple $\mX \subset L^2(\Omega) \subset \mX^*$ with $\Omega = \mathbb{T}^d$.
%     We seek to solve
%     \begin{equation}\label{eq:cauchy}
%         \partial_t u(t) = \mG(u(t)), \qquad u(0) = u_{\mathrm{init}} \in \mX,
%     \end{equation}
%     on the solution space
%     \begin{equation}
%         W(0, T; \mX) := \bigl\{u \in L^2(0, T; \mX) \bigm| \partial_t u \in L^2(0, T; \mX)\bigr\},
%     \end{equation}
%     where the nonlinear operator $\mG : \mX \to \mX$ admits the decomposition
%     \begin{equation}
%         \mG(u) = -\mL(u)\mu(u), \qquad \mL(u) = \mM(u) + \mW(u),
%     \end{equation}
%     with $\mu(u) = \delta V / \delta u \in \mX^*$ for a Fr\'{e}chet differentiable free-energy functional $V : \mX \to \mR$.
% \end{Prob}

The generalized Onsager principle \cref{eq:onsager-pde} provides a unifying framework for a diverse array of infinite-dimensional dynamics. Its scope encompasses many classical PDE dynamics, including purely dissipative systems (e.g., the Allen--Cahn and Cahn--Hilliard equations), purely conservative systems (e.g., the KdV and nonlinear Schr\"{o}dinger equations), and mixed dynamics typified by the incompressible Navier--Stokes equations. 
Since the generalized Onsager principle is formulated at the abstract operator level, it extends naturally to non-local integro-differential operators and the effective mesoscopic limits of discrete microscopic systems, such as the nonlinear particle chains explored later in our numerical experiments.
More details of these examples are exhibited in \cref{sec:examples}.

\subsection{Spectral representation}
\label{sec:spectral-representation}

A key consequence of \cref{ass:convolution} is that the operators $\mM(u)$ and $\mW(u)$ diagonalize in Fourier space. 

\begin{Thm}\label[Thm]{thm:spectral-representation}
    Under \cref{ass:convolution}, we define the Fourier series
    \begin{equation}
        \widehat\mM^{[k]}(u)=\Hinner{K_\mM(\cdot;u)}{e_k}\quad\text{and}\quad\widehat\mW^{[k]}(u)=\Hinner{K_\mW(\cdot;u)}{e_k},\quad\forall k\in\mZ
    \end{equation}
    for the kernels $K_\mM(\cdot;u)$ and $K_\mW(\cdot;u)$, respectively, where $\{e_k\}_{k \in \mathbb{Z}}$ denotes the Fourier basis of $L^2(\Omega)$. Then it holds for $\mM$ (and $\mW$, respectively) that
    \begin{equation}
        \sum_{|k|\le N}\widehat\mM^{[k]}(u)\hat\eta^{[k]}e_k\to\mM(u)
    \end{equation}
    in $\mX$ as $N\to+\infty$ and
    \begin{equation}
        \Xcouple{\xi}{\mM(u)\eta}=\sum_{k\in\mZ}\widehat\mM^{[k]}(u)\hat\eta^{[k]}\hat\xi^{[-k]}
    \end{equation}
    for any $\xi$ and $\eta\in\mX$, where we set    
    \begin{equation}
        \hat\xi^{[k]}=\Xcouple{\xi}{\overline{e_k}}\quad\text{and}\quad\hat\eta^{[k]}=\Xcouple{\eta}{\overline{e_k}},\quad\forall k\in\mZ.
    \end{equation}
    Furthermore, \Cref{ass:symmetry} holds if and only if
    \begin{equation}\label{eq:spectral-constraints}
        \widehat\mM^{[k]}(u) = \overline{\widehat\mM^{[-k]}(u)} \geq 0
        \qquad \text{and} \qquad
        \widehat\mW^{[k]}(u) = \overline{\widehat\mW^{[-k]}(u)} \in \mi\mR
    \end{equation}
    for all $k \in \mathbb{Z}$.
\end{Thm}
\begin{proof}
    By definition,
    \begin{equation}
        \begin{aligned}
            \Hinner{\mM(u)(\eta)}{e_k}&=\int_\Omega\Xcouple{\eta}{K_\mM(x-\cdot;u)}\overline{e_k}(x)\md x\\
            &=\Xcouple{\eta}{\int K_\mM(x-\cdot;u)\overline{e_k}(x)\md x}\\
            &=\Xcouple{\eta}{\overline{e_k}(\cdot)\int K_\mM(x-\cdot;u)\overline{e_k}(x-\cdot)\md x}\\
            &=\Xcouple{\eta}{\Hinner{K_\mM(u)}{e_k}\overline{e_k}}=\widehat\mM^{[k]}(u)\hat\eta^{[k]}.
        \end{aligned}
    \end{equation}
    The second equality holds because the operator $\mu$ is continuously linear and the mapping $x\mapsto \overline{e_k(x)}K_\mM(x-\cdot;u)$ is Bochner integrable due to the boundness of $\Omega=\mT^d$ and the $L^2$-integrability of $K_\mM$. Similar deduction can be applied to $\mW$ as well, and all the conclusions follow immediately by direct evaluation.
\end{proof}
This theorem lies at the heart of our proposed framework. By establishing that the stringent infinite-dimensional constraints of symmetry, positive semi-definiteness, and anti-symmetry (\cref{ass:symmetry}) are strictly equivalent to pointwise conditions on Fourier multipliers, the problem is rendered both analytically and computationally tractable. The spectral characterization also serves as the foundational bridge connecting the theoretical guarantees derived in \cref{sec:theoretical-analysis} with the practical network parameterization detailed in \cref{sec:network}.

\section{Theoretical analysis}\label{sec:theoretical-analysis}
In this section, we establish a rigorous theoretical foundation for the generalized Onsager dynamics \eqref{eq:onsager-pde}. Throughout our analysis, all mathematical proofs inherently rely on the structural hypotheses formalized in \cref{ass:symmetry,ass:convolution,ass:lipschitz-v2}. We begin by proving the global well-posedness of the system and characterizing the long-time asymptotic behavior of its solution trajectories. To guarantee that the learned operators carry genuine physical meaning rather than acting as unconstrained fitting devices, we then establish the identifiability of the Onsager factorization when $\mM$ and $\mW$ are independent of the state $u$. Finally, bridging the continuous formulation with its practical implementation, we derive explicit stability conditions that ensure the time-discrete approximation strictly preserves the energy dissipation property.

\subsection{Global well-posedness}
\label{sec:global-wellposedness}
% \begin{Thm}[Global well-posedness]\label[Thm]{thm:global-wellposedness}
%     The Cauchy problem \cref{eq:cauchy} admits a unique solution on $t\in[0, +\infty)$ that depends continuously on the initial condition.
% \end{Thm}
We begin with the most basic analytical requirement for the proposed hypothesis class by establishing global well-posedness under the structural hypotheses. 
% \subsection{Proof of \cref{thm:global-wellposedness}}

\begin{Thm}[Global existence]\label{thm:global-existence}
    % Suppose that \cref{ass:symmetry,ass:convolution,ass:lipschitz-v2,ass:X-constraints} hold.
    For any initial condition $u(0)=u_\init\in\mX$, the Cauchy problem \cref{eq:cauchy} admits a global solution $u$ defined on $[0, +\infty)$.
\end{Thm}
\begin{proof}[Proof sketch]
To construct a global solution, we first define a sequence of finite-dimensional approximate solutions $\{u_n\}_n$, establishing their local existence via Peano's theorem. We then leverage the energy dissipation property of the operator $\mL$ according to \cref{ass:symmetry,ass:convolution}, combined with the coercivity of the functional $V$, to derive a uniform bound for this sequence. Applying the Aubin--Lions lemma alongside a diagonal extraction process yields a weakly convergent subsequence. By invoking the Lipschitz continuity stipulated in \cref{ass:lipschitz-v2}, we demonstrate that this weak limit is indeed a global solution to the Cauchy problem. 
% Finally, uniqueness and continuous dependence on the initial condition follow directly from an application of Gr\"onwall's inequality. The detailed proof is exhibited as below.\tobediscussed{move the proof to SM?}
A detailed proof can be found in \cref{sec:proof-of-global-existence}.
\end{proof}
Let $J:\mX\to\mX^*$ denote the duality mapping, then the uniform smoothness of $\mX$ guarantees that $J$ is single-valued and uniformly continuous on bounded sets. The mapping $J$ can also be treated as the Fr\'echet derivative of $\frac12\|\cdot\|_\mX^2$. 
\begin{Lem}\label[Lem]{lem:J-norm}
    For any $u\in \mX$, the functional $J(u)\in\mX^*$ satisfies
    \begin{equation}
        \langle J(u),u\rangle_{\mX^*,\mX}=\|u\|_\mX^2,\quad\|J(u)\|_{\mX^*}=\|u\|_\mX.
    \end{equation}
\end{Lem}
The proof is deferred to \cref{sec:proof-of-J-norm}, leading to the following uniqueness result.
\begin{Thm}[Global uniqueness and continuous dependence]\label{thm:global-uniqueness-v2}
    The global solution described in \cref{thm:global-existence} is unique.
    Moreover, if $\tilde{u}(t)$ and $\bar{u}(t)$ are two solutions with initial conditions $\tilde{u}_{\mathrm{init}}$ and $\bar{u}_{\mathrm{init}}$ respectively, then
    \begin{equation}\label{eq:continuous-dependence}
        \norm{\tilde{u}(t) - \bar{u}(t)}_\mX^2 \leq \me^{2 L_4(R)t}\norm{\tilde{u}_{\mathrm{init}} - \bar{u}_{\mathrm{init}}}_\mX^2, \qquad \forall t > 0,
    \end{equation}
    where $R$ is a uniform bound on $\norm{\tilde{u}(t)}_{\mX}$ and $\norm{\bar{u}(t)}_{\mX}$, and $L_4(R)$ is the $\mX$-Lipschitz constant for $\mG$.
\end{Thm}

\begin{proof}
    By \cref{prop:energy-dissipation} and the coercivity of $V$, both trajectories are uniformly bounded in $\mX$.
    Let $R$ denote this common bound.
    Then
    \begin{equation}
    \begin{aligned}
        \ddt{}\|\tilde u(t)-\bar u(t)\|_\mX^2&=2\Xcouple{J(\tilde u(t)-\bar u(t))}{\partial_t\tilde u(t)-\partial_t\bar u(t)}\\
        &=2\Xcouple{J(\tilde u(t)-\bar u(t))}{\mG(\tilde u(t))-\mG(\bar u(t))}\\
        &\le 2\norm{J(\tilde u(t)-\bar u(t))}_{\mX^*}\norm{\mG(\tilde u(t))-\mG(\bar u(t))}_{\mX}\\
        &\le L_4(R)\|\tilde u(t)-\bar u(t)\|_\mX^2,
    \end{aligned}
    \end{equation}
    where the last inequality follows from \cref{lem:J-norm} and the $\mX$-Lipschitz condition of $\mG$.
    The conclusion follows by Gr\"{o}nwall's inequality.
    Uniqueness is the special case $\tilde{u}_{\mathrm{init}} = \bar{u}_{\mathrm{init}}$.
\end{proof}
\subsection{Asymptotic behavior}
\label{sec:asymptotic}
After establishing global well-posedness, we turn our attention to the long-time asymptotic behavior of the system. According to \cref{prop:energy-dissipation}, the free-energy functional $V$ acts strictly as a Lyapunov function for the dynamics. Consequently, the asymptotic limits of the trajectories are linked to the kernel structure of the dissipation operator $\mM$. To formalize this relationship, we distinguish between two critical invariant sets.
\begin{Def}[Equilibrium and zero-dissipation sets]\label{def:invariant-sets}
    We define the following subsets of $\mX$:
    \begin{enumerate}
        \item $\mathcal E$ as the set of equilibria, consisting of states where the system is stationary.
        \begin{equation}
            \mE = \{ \varphi \in \mX \mid \mG(\varphi) = -[\mM(\varphi) + \mW(\varphi)]\mu(\varphi) = 0 \}.
        \end{equation}
        \item $\mathcal Z$ as the zero-dissipation set, consisting of states where the free energy dissipation vanishes.
        \begin{equation}
            \mathcal{Z} = \Bigl\{ \varphi \in \mX \Bigm\vert \Xcouple{\mu(\varphi)}{\mM(\varphi)\mu(\varphi)} = 0 \Bigr\}.
        \end{equation}
    \end{enumerate}
    Clearly, $\mathcal E \subseteq \mathcal Z$.
\end{Def}
By virtue of the global well-posedness, the generalized Onsager dynamics induces a well-defined continuous flow $S(t) : u_\init\mapsto u(t)$. The following theorem demonstrates that every trajectory asymptotically approaches the zero-dissipation set, which, under specific strict positivity conditions, collapses exactly to the set of stationary equilibria.

\begin{Thm}\label[Thm]{thm:convergence-zero-dissipation}
    By global well-posedness in \cref{sec:global-wellposedness},
    the solution induces a continuous flow $S(t) : u(0) \mapsto u(t)$. For any trajectory $u(t)$, the $\omega$-limit set
    \begin{equation}
        \omega(u(0)) := \bigcap_{T > 0} \overline{\{u(t) : t \geq T\}}
    \end{equation}
    is a non-empty and compact subset of $\mathcal Z$, which is positively invariant under $S(t)$. Moreover, if the dissipation operator $\mM(u)$ satisfies the elliptic condition, which means there exists $c>0$ such that
        \begin{equation}\label{eq:strict-positivity}
            \Xcouple{\xi}{\mM(u)\xi} \geq c\norm{\xi}_{\mX^*}^2, \qquad \forall \xi \in \mX^*,
        \end{equation}
    then $\mathcal Z = \mE$, and any trajectory converges asymptotically to the set of static equilibria.
\end{Thm}
\begin{proof}
    A detailed proof can be found in \cref{sec:proof-convergence-zero-dissipation}.
    According to \cref{prop:energy-dissipation}, any trajectory $\{u(t)\}_{t\in[0,+\infty)}$ is uniformly bounded in $\mX$. By the Rellich--Kondrachov Theorem, the embedding $\mX=H^1(\Omega) \hookrightarrow L^2(\Omega)$ is compact. Therefore, we have a non-empty $\omega$-limit set $\omega(u(0))$.

We claim that the energy $V$ is a constant on the $\omega$-limit set. Otherwise there exist two sequence of snapshots with energy converging to distinct values, which contradict the monotonicity of $V$ along the trajectory.
Consider a trajectory lying on the $\omega$-limit set. By taking the time derivative of $V$, we can easily verify that the free energy dissipation vanishes, and consequently we have $\omega(u(0))\subseteq\mathcal Z$.
Furthermore, if the elliptic condition \cref{eq:strict-positivity} holds, then for any $\varphi \in \mathcal Z$, we have that $\Xcouple{\mu(\varphi)}{\mM(\varphi)\mu(\varphi)} = 0$ implies $\mu(\varphi) = 0$, so $\varphi \in \mE$.

\end{proof}
% \begin{proof}[Proof sketch]
%     \Cref{lem:energy-dissipation} and \cref{ass:coercivity} implies the boundness of the trajectory $\{u(t)\}_{t \geq 0}$ is uniformly bounded in $\mX$, and it follows by the compact embedding $\mX \hookrightarrow L^2(\Omega)$ that $\omega(u(0))$ is non-empty. Moreover, the monotonicity of $V$ on the trajectory ensures that $V$ is constant on $\omega(u(0))$. The proof details can be found in \cref{sec:proofs}.
% \end{proof}

In the general case where $\mM(u)$ has a nontrivial kernel, the zero-dissipation set $\mathcal Z$ may be strictly larger than the equilibrium set $\mE$.
Physically, this means that rather than settling into a static equilibrium, the system may converge to a persistent dynamic regime such as motion along a Hamiltonian orbit on a constant-energy surface driven entirely by the conservative operator $\mW$.

The significance of the convergence result lies in its strictly \emph{a priori} nature. We do not rely on the data-driven optimization process to discover, approximate, or enforce long-term stability; instead, up to fitting and discretization errors, any learned model is mathematically ensured to converge to a physically meaningful orbit. Consequently, the discovered mesoscopic dynamics remain physically interpretable and structurally sound over infinite time horizons.

\begin{Def}\label[Def]{def:elliptic-conservative}
    In fact, we have two physically distinct regimes of the generalized Onsager dynamics  \cref{eq:cauchy}:
    \begin{enumerate}
        \item \emph{Elliptic dissipative systems} (e.g., Allen--Cahn): $\mM$ satisfies the elliptic condition as in \cref{eq:strict-positivity}, the free energy decays strictly, and trajectories converge to static equilibria.
        \item \emph{Conservative systems} (e.g., KdV and other Hamiltonian PDEs): $\mM \equiv 0$, the free energy is exactly conserved, and trajectories may exhibit persistent dynamics on constant-energy surfaces.
    \end{enumerate}
\end{Def}
While intermediate regimes naturally exist where $\mM$ is non-zero but possesses a non-trivial kernel, violating the elliptic condition, a detailed theoretical treatment of such mixed dynamics falls outside the scope of the present study. Instead, we focus on the two principal extremes defined above.
Crucially, as we demonstrate in the subsequent section,
these regimes dictate entirely different identifiability properties for the learned Onsager factorization.

\subsection{Uniqueness of the Onsager factorization}\label{sec:factorization}
% Recall that a central objective of our framework is to recover \emph{interpretable} mesoscopic mechanisms. However, there is an inherent risk of degeneracy in the inverse problem: multiple distinct combinations of operators ($\mM$ and $\mW$) and free energy functionals $V$ could theoretically produce the exact same dynamics. If this factorization were highly non-unique, the learned components would act merely as a phenomenological black box rather than faithfully representing the true dissipative and conservative physics of the system. To address the physical identifiability of the learned dynamics, we demonstrate that when the operators $\mM$ and $\mW$ are state-independent, the generalized Onsager factorization is uniquely determined under mild spectral irreducibility conditions.
Recall that a central objective of our framework is to recover \emph{interpretable} mesoscopic mechanisms. Hence, it would be desirable if the operators ($\mM$ and $\mW$) and free energy functionals $V$ could be uniquely identified. While in general there may be a gauge dependence as shown in \cite{zhu2025identifiable}, we demonstrate that when the operators $\mM$ and $\mW$ are state-independent, the generalized Onsager factorization \cref{eq:onsager-pde} is uniquely determined under mild spectral irreducibility conditions.
\begin{Def}[spectral irreducibility]
    Consider a functional $V:\mX\to\mR$.
    We say that the functional $V$ is irreducible on the spectral space if it is second-order Fr\'echet differentiable, and there exists no partition $\mZ=S_1\sqcup S_2$ such that
    \begin{equation}
        \frac{\delta^2V}{\delta u^2}(e_p,e_q)=\left.\frac{\md^2}{\md\tau_1\md\tau_2}V(u+\tau_1e_p+\tau_2e_q)\right|_{\tau_1=\tau_2=0}\equiv0,\quad\forall p\in S_1,q\in S_2,
    \end{equation}
    where $\{e_k\}_{k\in\mZ}$ is the spectral basis of $\mH=L^2(\Omega)$.
\end{Def}
To build intuition for the spectral irreducibility condition, it is sufficient to note that nonlinear interaction within the integrand of the free energy potential $V$ will probably mix the spectral modes. For instance, the potential $V$ for the KdV dynamics has a cubic component $u^3$ within the integrand, whose computation mixes the spectral modes and thus the functional naturally satisfies the irreducibility requirement.

\begin{Thm}\label[Thm]{thm:factorization-uniqueness}
    Let $(\mM_1,\mW_1)$ and $(\mM_2,\mW_2)$ be independent of $u$ and satisfy 
    \begin{equation}\label{eq:2factorization}
        (\mM_1+\mW_1)\mu_1(u)=(\mM_2+\mW_2)\mu_2(u),\quad \mu_1=\frac{\delta V_1}{\delta u},\quad\mu_2=\frac{\delta V_2}{\delta u}
    \end{equation}
    with our previous hypotheses. Supposing that $V_2$ is irreducible on the spectral space, we have
    \begin{itemize}
        \item if $\mM_1$ is positive definite and $V_2$ is not a constant functional, then there exists a constant $\lambda>0$ such that
    \begin{equation}
        \mM_2=\lambda\mM_1,\quad\mW_2=\lambda\mW_1,\quad \mu_1=\lambda\mu_2;
    \end{equation}
        \item if $\mM_2=0$ and $\hat\mu_2^{[k]}\not\equiv0$ for all $k\in\mZ$, then $\mM_1\mu_1(u)\equiv0$. Furthermore, if $\ker\mW_1=0$, then $\mu_1=\lambda'\mu_2$ for a constant $\lambda'\in\mR$.
    \end{itemize}
\end{Thm}
The detailed proof can be found in \cref{sec:proof-factorization-uniqueness}.
\begin{Rem}
Crucially, the irreducibility requirement is a natural and mild assumption rather than a restrictive limitation. If a functional $V$ is reducible, it physically implies that the state variable can be partitioned into two or more mutually non-interacting subsystems. In such cases, the dynamics can simply be decomposed into their respective irreducible components. The uniqueness result of \cref{thm:factorization-uniqueness} can then be applied independently to each block. Consequently, the generalized Onsager factorization remains unique up to a set of independent, one-dimensional scaling factors, with exactly one scaling constant corresponding to each decoupled physical subsystem.
\end{Rem}

\Cref{thm:factorization-uniqueness} offers identifiability guarantees for the generalized Onsager principle when $\mM$ and $\mW$ are state-independent. Interpreting the left-hand side of \cref{eq:2factorization} as the learned dynamics and the right-hand side as the reference, two physically distinct regimes emerge.
\begin{enumerate}
    \item \emph{Elliptic dissipative systems}: Every learned factorization reproducing the reference dynamics recovers $\mM$, $\mW$, and $\mu$ simultaneously, each up to a common positive scalar. The dissipative and conservative operators are therefore individually identifiable, and the learned energy is an affine transformation of the true one.
    \item \emph{Conservative systems}: \cref{thm:factorization-uniqueness} forces the learned dissipation contribution $\mM_1\mu_1$ to vanish identically, regardless of $\mM_1$ itself. If additionally $\ker\mW_1=0$, the learned functional derivative coincides with the true one up to a real scalar, so the learned energy is conserved along every trajectory. In particular, the framework cannot spuriously inject dissipation into a genuinely conservative system.
\end{enumerate}
In both regimes, the Onsager factorization is unique modulo a one-dimensional rescaling.
% ---the minimal ambiguity consistent with the homogeneous structure of \cref{eq:2factorization}. 
For conservative systems an additional freedom remains in $\mM_1$, but it is physically inert since its contribution to the dynamics vanishes. These guarantees ensure that the learned energy functional and transport operators carry genuine physical meaning rather than serving as unconstrained fitting devices, thereby providing a rigorous foundation for interpreting the outputs of {\ournet} as discovered governing laws. Some experimental evidence for the identifiability can be found in \cref{sec:identifiability-of-dynamics}.

\subsection{Discrete energy dissipation}\label{sec:discrete}
While the continuous formulation guarantees thermodynamic consistency, practical implementation requires evaluating the system on a computational grid and advancing it via a numerical integrator. In our framework, we employ the explicit forward Euler method, where continuous stability does not automatically ensure numerical stability due to the discrepancies introduced by spatial and temporal discretization.

Let $\{u_p^s\}_{p\in G}$ denote the discrete state at time level $s$, where $u_p^s = u_p(s\Delta t)$ and $p$ indexes the grid points in a fixed grid $G$. We consider the time-discrete scheme
\begin{equation}\label{eq:discrete-update}
    u^{s+1} = u^s - \Delta t [\mM(u^s) + \mW(u^s)] \mu(u^s).
\end{equation}
Our goal is to identify conditions on the time step $\Delta t$ under which the discrete evolution inherits an energy-dissipation property. The estimate is obtained by applying the descent lemma to the discrete increment from $u^s$ to $u^{s+1}$.

% In practice, the physical features are stored in grid points with a fixed resolution, and the numerical integrator is chosen to be the explicit forward Euler. So the continuous stability does not automatically ensures numerical stability due to the discrepancy resulting from the discretization in space and time.

% Note that in the discrete case, the system is initialized as a finite-dimensional vector $\{u_p^0\}_{p\in G}$, and our goal is to continuously maintain and update the system state $\{u_p^k\}_{p\in G}$ for each time step $s$, where we use the abbreviation $u_p^s=u_p(s\Delta t)$ and let $p$ stand for a certain spatial coordinate from a fixed grid $G$.
% We derive a condition on the time step $\Delta t$ to prevent numerical divergence for the discrete update
% \begin{equation}\label{eq:discrete-update}
%     u^{s+1} = u^s - \Delta t [\mM(u^s) + \mW(u^s)] \mu(u^s).
% \end{equation}
% The decrease of energy can be bounded by applying the Descent Lemma to the step from $u^s$ to $u^{s+1}$.
\begin{Thm}[Discrete energy dissipation]\label{thm:discrete-stability}
The discrete evolution satisfies the energy dissipation property $V(u^{s+1}) \le V(u^s)$ if the time step $\Delta t\le2\kappa L_1(R_0)^{-1}$, where
    \begin{equation}
        \kappa=\frac{2\Xcouple{\mu^s}{\mM^s\mu^s}}{\norm{(\mM^s+\mW^s)\mu^s}_\mX^2}
    \end{equation}
    and $R_0$ depends only on $V$.
    Here, we use the abbreviations $\mM^s\equiv\mM(u^s)$, $\mW^s\equiv\mW(u^s)$, and $\mu^s\equiv\mu(u^s)$. Moreover, 
    \begin{itemize}
        \item for elliptic dissipative dynamics where \cref{eq:strict-positivity} holds, the parameter $\kappa\ge\kappa_0>0$, and the constant $\kappa_0$ depends only on $\mM$, $\mW$, and $V$.
        \item for conservative dynamics where $\mM=0$, the one-step variation of the potential
        \begin{equation}
            \abs{V(u^{s+1})-V(u^s)}\le C(\tilde R)(\Delta t)^2=\mathcal O(\Delta t^2),
        \end{equation}
        where the parameter $\tilde R$ depends on $\mW$, $V$, and the maximal $\mX$-norm across the discrete trajectory. It follows that the potential changes $\abs{V(u^s)-V(u^0)}$ on the whole trajectory is $\mathcal O(\Delta t)$
        % \begin{equation}\label{eq:Vus-increment}
        %     \left|V(u^s)-V(u^0)\right|=\mathcal O(\Delta t)
        % \end{equation}
        for any $s\Delta t\le T$ if the trajectory does not blow up within finite time $T$.% The parameter $\tilde R$ depends on $\mW$, $V$, and the maximal $\mX$-norm across the discrete trajectory.
    \end{itemize}
        % The discrete evolution \cref{eq:discrete-update} satisfies the energy dissipation property $V(u^{s+1}) \le V(u^s)$ if the time step $\Delta t\le2\kappa L_1(R_0)^{-1}$, where $\kappa$ is determined by the dynamics ($\mM$, $\mW$, and $V$) and the current snapshot $u^s$. $R_0$ depends only on $V$. Moreover,
        % \begin{itemize}
        %     \item for elliptic dissipative systems, the parameter $\kappa$ has a positive lower bound $\kappa_0$ independent of the solution state $u^s$;
        %     \item for conservative systems, the one-step variation of potentials for a discrete trajectory is of order $\mathcal O(\Delta t^2)$ %, and the change of potentials on the whole trajectory is of order $\mathcal O(\Delta t)$
        %     if the trajectory does not blow up within finite time.
        % \end{itemize}
\end{Thm}
\begin{proof}[Proof sketch]
By the Lipschitz continuity of $\mu_1$, applying the descent lemma we have
\begin{equation}\label{eq:Descent-Lemma}
        V(u^{s+1}) - V(u^s) \le \Xcouple{\mu^s}{\mM^s\mu^s}\Delta t + \mathcal O(\Delta t^2)
\end{equation}
as long as
\begin{equation}\label{eq:R-range}
    \max\left(\norm{u^s}_\mX,\norm{u^{s+1}}_\mX\right)\le R.
\end{equation}
% Combining with the discrete update \cref{eq:discrete-update}, we may conclude that it suffices to ensure that
% \begin{equation}\label{eq:L1R-sufficient}
%      \left\|[\mM(u^s) + \mW(u^s)] \mu(u^s)\right\|_\mX^2L_1(R)\Delta t\le 2\left\langle \mu(u^s), \mM(u^s) \mu(u^s) \right\rangle_{\mX^*,\mX},
% \end{equation}
% for some $R$ satisfying \cref{eq:R-range}. Note that in \cref{eq:R-range}, $R$ is partially determined by $\|u_{n+1}\|_\mX$, and we have to dismiss the dependency. 
Then, we can prove by induction that either $\mu^s=0$, or
\begin{equation}\label{eq:claim}
    \Delta t\le\frac{2\Xcouple{\mu^s}{\mM^s\mu^s}}{\norm{(\mM^s+\mW^s)\mu^s}_\mX^2L_1(R^s)}\quad\text{and}\quad R^s=D+\frac{2\Xcouple{\mu^s}{\mM^s\mu^s}}{\norm{(\mM^s+\mW^s)\mu^s}_\mX L_1(D)}
\end{equation}
for the $s$th step will ensure that $V(u^{s+1})\le V(u^s)$ and $\norm{u^s}_\mX\le D$ for all $s$. The constant $D$ is chosen such that
\begin{equation}
    V(w)\le V(u^0)\implies\|w\|_\mX\le D
\end{equation}
for any $w\in\mX$, and the existence is guaranteed by the coercivity of $V$.
The conclusion for elliptic dissipative dynamics and conservative dynamics follows respectively by direct deduction.
A detailed proof can be found in \cref{sec:proof-discrete-energy-dissipation}.
% \tobediscussed{move to the SM part?}
    % \input{proofs/discrete-energy-dissipation-simple}
\end{proof}
% This theorem shows that for elliptic dissipative systems, the Euler integrator can always lead to non-increasing free energy functional $V$ as long as the time step size is small enough. For conservative systems, the potential along a discrete trajectory without blow-up is conserved up to a scale of $\Delta t$.
Consequently, for elliptic dissipative systems, the explicit forward Euler scheme strictly preserves the monotonic decay of the free energy when the time step is sufficiently small. For conservative systems,  while exact energy preservation is relaxed due to time discretization, the discrete energy fluctuates by at most an $\mathcal O(\Delta t^2)$ for each update.

This result reflects a fundamental philosophical departure from the traditional scientific computing and numerical analysis workflows. In classical settings, ensuring long-term stability and thermodynamic consistency typically requires the rigorous design of complex, equation-specific discretizations. In contrast, our framework reverses this by implicitly imposing these stability constraints by defining a rigorous continuous hypothesis class prior to any data-driven identification of equations.
Consequently, even the simplest explicit integrator automatically inherits these robust physical guarantees. This decoupled approach not only provides critical flexibility for neural network model design, but also enables a unified, computationally efficient temporal discretization that applies universally across an entire class of unknown, data-driven mesoscopic dynamics without requiring instance-specific numerical engineering.
\section{\ournet}\label{sec:network}
Starting from this section, we focus our discussion on a specific neural network implementation of the framework,
where we set $\mX=H^1(\Omega)$. We name the architecture \ournet, abbreviated as \ournetabbrv.
To construct a network for learning the generalized Onsager principle, we take the ansatz
\begin{equation}\label{eq:V-ansatz}
    V(u) = \underbrace{\int_\Omega\left[\frac{\alpha}{2}|u|^2+F(u(x))\right]\md x}_{V_0(u)}+\underbrace{\int_\Omega \frac{\beta}{2} |\nabla u|^2\md x}_{V_1(u)}+V_2(u)
\end{equation}
for the free energy $V$ with positive scalars $\alpha$ and $\beta$. Here, $V_0$ and $V_1$ account for pointwise potentials and local interactions, and $V_2:\mX\to\mR$ is a functional capturing the residuals of the integrands. 
The motivation for this explicit decomposition comes from physics. In fact, the ansatz mirrors the standard structure of thermodynamic free energies: the $\beta$-term penalizes steep spatial gradients to account for interfacial energy, capillarity, or dispersion when coupled with a conservative transport operator; $\alpha$ provides a global quadratic stabilizing potential that ensures the coercivity of the functional; the pointwise function $F(u)$ models local energy contributions; and $V_2(u)$ serves as a flexible residual functional designed to capture complex, non-local, or unknown mesoscopic interactions.

We then obtain the functional derivative
\begin{equation}
    \mu(u)=\frac{\delta V}{\delta u}=\alpha u-\beta\Delta u+F'(u)+\frac{\delta V_2}{\delta u}.
\end{equation}
To ensure that the preceding theoretical analysis applies to our ansatz \cref{eq:V-ansatz}, we must impose specific constraints on the functional $V$. More specifically, we let $\alpha, \beta > 0$ be positive scalars and define
\begin{equation}
    Q(u)=\int_\Omega F(u(x))\md x+V_2(u).
\end{equation}
We further enforce that the functional $Q:\mX\to\mR$ is bounded from below and that its Fr\'echet derivative $\delta Q/\delta u$ is $\mX$-Lipschitz continuous on bounded set.
These conditions guarantee that both the coercivity requirement in \cref{ass:symmetry} and the $\mX$-Lipschitz continuity in \cref{ass:lipschitz-v2} are satisfied.

For network implementation, our network is designed to follow the update \cref{eq:discrete-update} in spectral space.
More specifically, we consider the explicit Euler step
% \begin{equation}
%     \hat u^{s+1} =\hat u^s - \Delta t \left[\widehat\mM(u^s) + \widehat\mW(u^s)\right]\odot \mu(u^s),
% \end{equation}
% or equivalently,
\begin{equation}
    \hat u^{s+1,[k]} =\hat u^{s,[k]} - \Delta t \left[\widehat\mM^{[k]}(u^s) + \widehat\mW^{[k]}(u^s)\right]\hat\mu^{s,[k]}
\end{equation}
for each spectral mode, where we define $\hat\mu^{s,[k]}$ as the $k$th mode for $\mu(u^s)$. Since
\begin{equation}\label{eq:spectral-of-derivative}
\begin{aligned}
    \hat\mu^{[k]}:=\Xcouple{\mu(u)}{\overline{e_k}}&=\Xcouple{\alpha u-\beta\Delta u+F'(u)}{\overline{e_k}}+\Xcouple{\frac{\delta V_2}{\delta u}}{\overline{e_k}}\\
    &=\left[\alpha+(2\pi k)^2\beta\right]\hat u^{[k]}+\widehat{F'(u)}^{[k]}+\overline{\left(\frac{\partial v_2}{\partial\hat u}\right)}^{[k]},
\end{aligned}
\end{equation}
where $v_2$ is the spectral representation of $V_2$, we have to specify appropriate parameterizations for the terms $\mM$, $\mW$, $F$, and $v_2$.

\paragraph{Convolutions $\mM$ and $\mW$}
To enforce the constraints in \cref{ass:symmetry}, it suffices, by \cref{thm:spectral-representation},  to require each entry of $\widehat\mM_\psi(u^s)$ to be a non-negative real number and each entry of $\widehat\mW_\psi(u^s)$ to be purely imaginary .
Consequently, we parameterize $\widehat\mM$ and $\widehat\mW$ as
\begin{equation}
    \widehat\mM_{\psi}(u) = \Re\left(G_\psi(\hat u)\right)^2\quad\text{and}\quad\widehat\mW_{\psi}(u) = \mi\Im\left(G_\psi(\hat u)\right),
\end{equation}
where $\Re(\cdot)$ and $\Im(\cdot)$ denote the real and imaginary part of the network output $G_\psi(\hat u)$.

\paragraph{Free energy functional $V$}
With the ansatz given by \cref{eq:V-ansatz}, we directly parameterize the functional variation $\mu=\delta V/\delta u$ as
\begin{equation}
    \mu_\phi(u) = \alpha u-\beta\Delta u+F_\phi'(u)+\frac{\delta V_\phi}{\delta u},
\end{equation}
then according to \cref{eq:spectral-of-derivative}, the spectral modes of $\mu$ are written as
\begin{equation}\label{eq:S-OnsagerNet-update}
    \hat\mu_\phi(u)=\left[\alpha+(2\pi \mathbf k)^2\beta\right]\odot\hat u+\widehat{F_\phi'(u)}+\overline{\nabla v_\phi(\hat u)},
\end{equation}
where \(\mathbf{k}\) is the Fourier mode-index vector, with the entry corresponding to mode \(k\) equal to \(k\).
% where $\mathbf{k}$ is a vector $(a_k)_{k\in\mZ}$ with $a_k=k$ for any $k$.
As explained previously, to make our theoretical analysis applicable for our network, we only need to force $\alpha,\beta>0$ with a softplus layer and control the boundedness of the networks since the remaining continuity assumptions are automatically satisfied for common network architectures. In particular, we adopt a sinusoidal activation function for the final layers. The Fourier spectral modes are computed with the discrete Fourier transform, and the gradient is evaluated by auto-differentiation.

In summary, we use the update
\begin{equation}\label{eq:parameterized-update}
    \hat u^{s+1}\approx\mathcal T_\theta(\hat u^s):=\hat u^s-\Delta t\left[\Re\left(G_\psi(\hat u)\right)^2 + \mi\Im\left(G_\psi(\hat u)\right)\right]\odot\hat\mu_\phi(u^s),
\end{equation}
where
\begin{equation}\label{eq:parameterization-mu}
    \hat\mu_\phi(u^s)=\left[\alpha+(2\pi \mathbf{k})^2\beta\right]\odot\hat u^s+\widehat{F_\phi'(u^s)}+\overline{\nabla v_\phi(\hat u^s)}.
\end{equation}
The parameter family $\theta$ contains all the trainable parameters $\psi$, $\phi$, $\alpha$, and $\beta$.
\section{Numerical experiments}\label{sec:experiments}
We empirically validate the proposed framework in two complementary settings: (i) learning from data generated from PDE models, where the underlying dynamics are known and serve as ground truth; and (ii) learning effective mesoscopic dynamics directly from microscopic simulation data, where no closed-form macroscopic equation is available a priori. The first setting allows us to rigorously assess predictive accuracy, long-time stability, and the identifiability properties established in \cref{sec:theoretical-analysis}, while the second setting demonstrates that our framework can serve as a practical tool for discovering interpretable mesoscopic models from microscopic data.
% \zhuoyuan{we need to give a brief description and a reason why we want to shift attention from PDE cases to microscopic cases}
% Our numerical experiments evaluate the proposed framework across four distinct datasets, comprising two macroscopic continuum models (the KdV and Allen--Cahn equations) and two microscopic particle systems (the FPU and FENE chain models). High-fidelity reference trajectories for the continuum PDEs are simulated via a spectral second-order semi-implicit backward differentiation formula (SBDF2), while the microscopic chain models are generated by the velocity Verlet algorithm. Across all configurations, we fix the length of trajectories as $100$ for subsequent network training and evaluation. Details for the data generation are described in \cref{sec:data-generation}.
\subsection{Network training}
The evolution of the system is solved in the spectral domain using a simple explicit forward Euler scheme. To optimize the network parameters $\theta$, we compute the reconstruction error for each reference trajectory $\{u^s\}_{s=1}^S$ using a $K$-step accumulated loss function, evaluated as
% and the regularization error
\begin{equation}
    \mJ(\theta) = \sum_{s=1}^{S-K}\sum_{k=1}^K \norm{\mathcal T_\theta^k(\hat u^s)^\vee - u^{s+k}}_2^2,
\end{equation}
where $\mathcal T_\theta$ represents the Euler update specified in  \cref{eq:parameterized-update}. The symbol ``$\vee$'' denotes the inverse discrete Fourier transform, so that the norm is computed in physical space. We adopt the multi-step loss specifically to penalize accumulating integration errors and enhance long-term predictive stability, but importantly, our framework is not limited to this exact learning approach. Any other objective capable of training the dynamics, such as the detaching trick used in \cite{brandstetter2022message,List2025}, can be applied without compromising the rigorous thermodynamic structure.

\subsection{Learning from known PDE dynamics}
We first evaluate the predictive capabilities and numerical stability of the proposed framework on two known PDE dynamics: the Allen--Cahn model (both 1D and 2D cases), representing purely dissipative phase-field dynamics, and the KdV model, representing conservative dispersive wave propagation. Readers may refer to \cref{sec:PDE-models} for more details about these PDE models.

\subsubsection{Predictive accuracy and long-term stability}
To demonstrate the capability of \ournetabbrv, we benchmark our model with 3 different baselines: (i) A classical one-step numerical scheme analogous to that for generating the reference trajectory, but with a larger ($\times25$) step size for a fair comparison; (ii) A data-driven approach based on the Fourier Neural Operator (FNO) \cite{li2020fourier} without any structural assumptions; (iii) The original OnsagerNet \cite{Yu2021OnsagerNet} designed for finite-dimensional cases, where we fix the spatial discretization and restate the learning task in the discrete space. The general workflow of our model is displayed in \cref{fig:architecture}. Detailed description about the baselines as well as the evaluation metric can be found in \cref{sec:baselines,sec:evaluation-metric}.

\begin{figure}
    \centering
    \includegraphics[width=\linewidth]{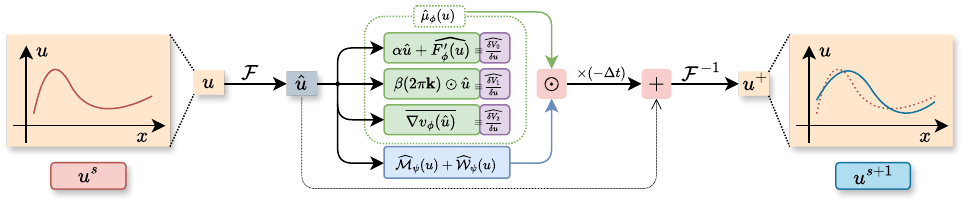}
    \caption{The general workflow for our \ournet.}
    \label{fig:architecture}
\end{figure}

As illustrated in \Cref{fig:kdv-ac-pred-error}, the rigorous enforcement of the generalized Onsager structure of our model serves as a highly effective inductive bias. Compared to the FNO model as a purely black-box learning approach and the OnsagerNet model as a naive generalization to PDE cases, {\ournetabbrv} demonstrates superior predictive accuracy and robust stability over extended integration horizons. 
Note that we do not apply the classical OnsagerNet to the 2D Allen--Cahn model due to its impractical memory cost since we would need around 100 times more budgets than those for the 1D case for a $128\times128$ grid.
The advantage of our {\ournetabbrv} is also quantitatively confirmed in \cref{tab:benchmark}, where {\ournetabbrv} consistently achieves the lowest 5-step relative prediction error with the minimal number of parameters. 
% The spatiotemporal reconstructions of the KdV dynamics, presented in \Cref{fig:kdv-spacetime-viz}, visually corroborate these error metrics. 
% {\ournetabbrv} captures the complex soliton interactions and phase shifts with exceptional fidelity, whereas the FNO predictions rapidly degrade, accumulating visible non-physical artifacts over time.

% We first test the proposed method on two representative examples, namely the Allen--Cahn equation and the KdV equation. 
% \Cref{fig:kdv-ac-pred-error} indicates that the Onsager structure serves as an effective inductive bias. 
% Relative to FNO, taken here as a representative data-driven baseline, S-OnsagerNet delivers higher predictive accuracy and improved stability.
% Compared with four baseline methods, S-OnsagerNet attains the smallest 5-step relative prediction error, as reported in \Cref{tab:benchmark}, and consistently delivers better accuracy.
% The KdV spatiotemporal reconstruction results in \Cref{fig:kdv-spacetime-viz} further support this observation. 
% S-OnsagerNet captures the reference patterns with much higher fidelity, while FNO produces visible artifacts in several instances.
\begin{figure}
    \centering
    \begin{minipage}{.48\textwidth}
    \centering
    \includegraphics[width=\linewidth]{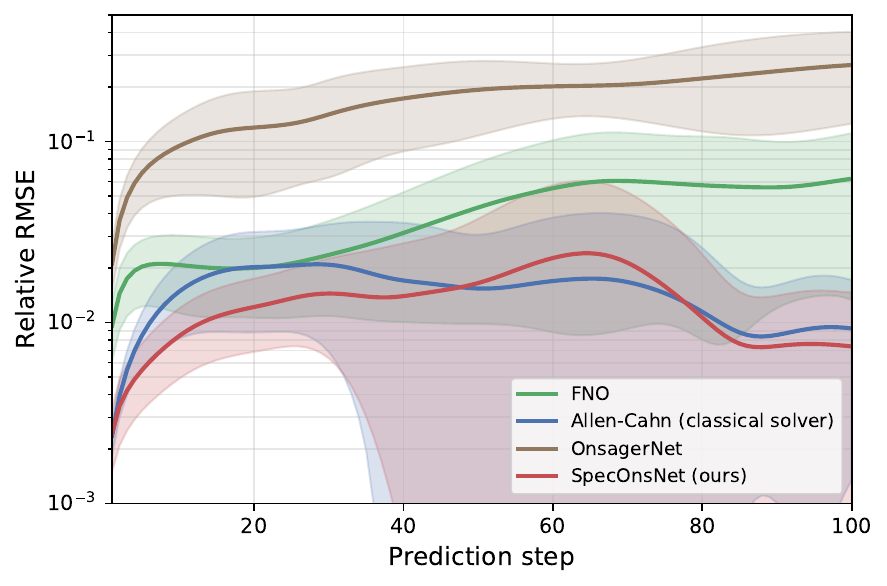}
    \subcaption{Allen--Cahn (1D) dataset}
    \end{minipage}%
    \hfill
    \begin{minipage}{.48\textwidth}
    \centering
    \includegraphics[width=\linewidth]{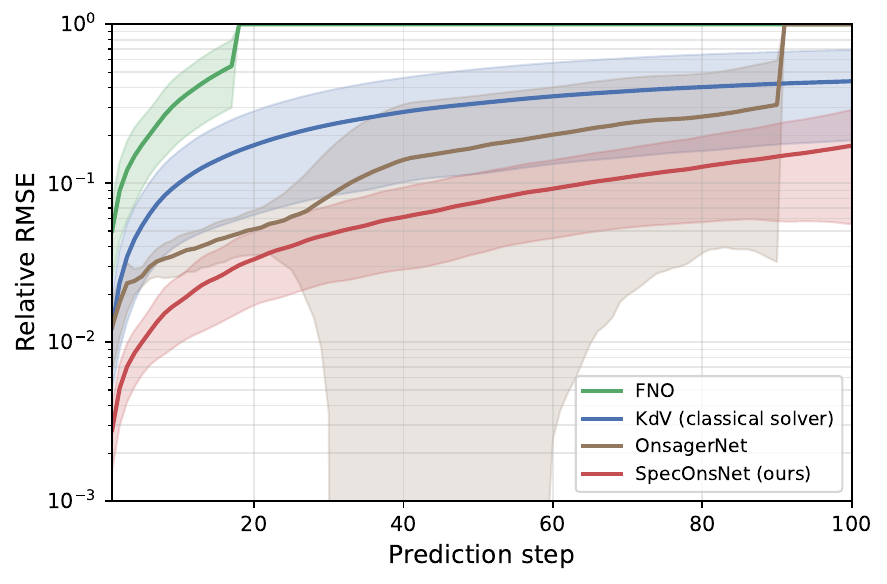}
    \subcaption{KdV dataset}
    \end{minipage}
    \caption{Comparison of long-time prediction performance}
    \label{fig:kdv-ac-pred-error}
\end{figure}

\begin{table}[h]
    \centering
    \resizebox{\textwidth}{!}{%
    \begin{tabular}{l|c|c|c|c}
    \toprule
    Task & Classical solver & FNO\cite{li2020fourier} & OnsagerNet\cite{Yu2021OnsagerNet} & \ournetabbrv \\
    \midrule
    \#params (1D; 2D) & --- & (66.3 K; 1.3 M) & (1.3 M; N/A) & (\textbf{18.2 K}; \textbf{50.0 K}) \\
    \midrule
    Allen--Cahn (1D) & ${0.0085}^{\pm0.0029}$ & $0.0204^{\pm0.0081}$ & $0.0260^{\pm0.0037}$ & $\mathbf{0.0055}^{\pm0.0022}$ \\
    KdV & $0.0541^{\pm0.0313}$ & $0.1747^{\pm0.0820}$ & $0.0671^{\pm0.0241}$ & $\mathbf{0.0100}^{\pm0.0037}$ \\
    Allen--Cahn (2D) & $0.0074^{\pm0.0008}$ & $0.0311^{\pm0.0047}$ & Out of Memory & $\mathbf{0.0037}^{\pm0.0011}$ \\
    \midrule
    FPUT chain & --- & $0.0360^{\pm0.0096}$ & $0.0242^{\pm0.0075}$ & $\mathbf{0.0082}^{\pm0.0026}$ \\
    FENE chain & --- & $0.0186^{\pm0.0066}$ & $0.0139^{\pm0.0076}$ & $\mathbf{0.0012}^{\pm0.0004}$ \\
    \bottomrule
    \end{tabular}%
    }
    \caption{5-step relative prediction error ($\downarrow$) for various baselines.}
    \label{tab:benchmark}
\end{table}

\subsubsection{Identifiability of the dynamics}\label{sec:identifiability-of-dynamics}
% directly validates uniqueness theorems: affine V for Allen–Cahn per Prop 4.19, conserved V for KdV/FENE per Prop 4.20)

Beyond mere trajectory matching, a core objective of our framework is the faithful extraction of the underlying thermodynamic mechanisms. To this end, we investigate the temporal behavior of the learned potentials evaluated on the reference trajectories to validate the identifiability guarantees established in \cref{sec:factorization}. 
As depicted in \cref{fig:kdv-ac-potential-evolution}, for the dissipative Allen--Cahn system, the learned free-energy functional $V_\theta$ exhibits a strictly affine relationship (highlighted by the dashed linear fit) with the true physical potential $V$. This empirical result corroborates the theoretical prediction that, for elliptic dissipative systems with $\mM$ and $\mW$ independent of the state, the Onsager factorization is identifiable up to an affine transformation. Meanwhile, for the purely conservative KdV system, the learned dynamics strictly conserve the learned potential. This confirms our theoretical assertion that the framework cannot spuriously inject artificial dissipation into a genuinely conservative system. To sum up, the experimental results on the PDE cases have numerically confirmed the identifiability of the proposed generalized Onsager principle shown in \cref{thm:factorization-uniqueness}.
% Ultimately, these results demonstrate that the extracted operators possess genuine physical interpretability, acting as discovered governing laws rather than unconstrained fitting artifacts.

% For the Allen–Cahn dataset, the learned potential shows an almost linear relation with the true potential, as highlighted by the dashed linear fit. For the KdV dataset, the learned dynamics approximately conserve the learned potential and the real potential.

\begin{figure}
    \centering
    \begin{minipage}{.48\textwidth}
    \centering
    \includegraphics[width=\linewidth]{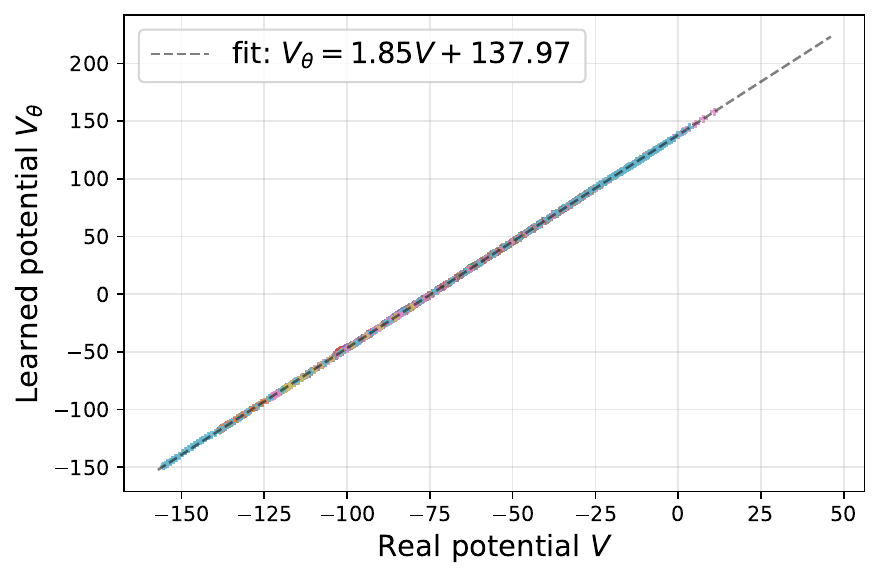}
    \subcaption{Allen--Cahn (1D) dataset}
    \end{minipage}%
    \hfill
    \begin{minipage}{.48\textwidth}
    \centering
    \includegraphics[width=\linewidth]{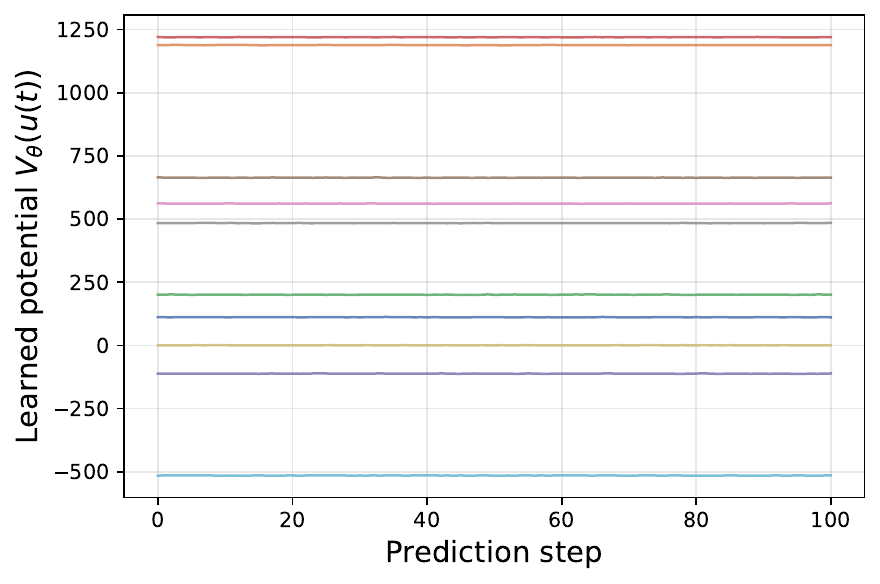}
    \subcaption{KdV dataset}
    \end{minipage}
    \caption{Identifiability of the PDE dynamics. (a) relationship between the learned potential and the real potential, and each point corresponds to one snapshot; (b) evolution of the learned potentials for the KdV dataset. For both the two figures, we use colors to distinguish 10 different trajectories in the test dataset.}
    \label{fig:kdv-ac-potential-evolution}
\end{figure}

\subsubsection{Empirical validation of a priori discrete stability}
In \cref{sec:theoretical-analysis}, we have established rigorous \textit{a priori} guarantees that apply universally to any dynamics within our hypothesis class. An important consequence of this framework, formalized in \cref{thm:discrete-stability}, is the preservation of discrete energy dissipation. Specifically, for elliptic dissipative systems, the explicit numerical integration strictly enforces a monotonically non-increasing free energy, guaranteeing trajectory stability provided the inference step size remains below a critical threshold. To empirically validate this theoretical bound, we evaluate the inference dynamics of both 1D and 2D Allen--Cahn models. As illustrated in \cref{fig:dissipation-dt-scaling}, the observed behavior aligns with our theoretical predictions. When the inference time step exceeds the characteristic scale used during training ($3\times10^{-3}$ compared to $1\times10^{-3}$), the predicted trajectory diverges, and the strict monotonic decay of the learned potential is subsequently violated.
Conversely, selecting an inference step smaller than $1\times10^{-3}$ safely preserves the monotonic decay of the free energy, leading to robust, long-term trajectory stability without any degradation in relative error. 
These observations empirically confirm the discrete stability of our {\ournetabbrv}, consistent with the theoretical guarantee induced by the structurally constrained hypothesis class.
\begin{figure}
    \centering
    \begin{minipage}{.48\textwidth}
    \centering
    \includegraphics[width=\linewidth]{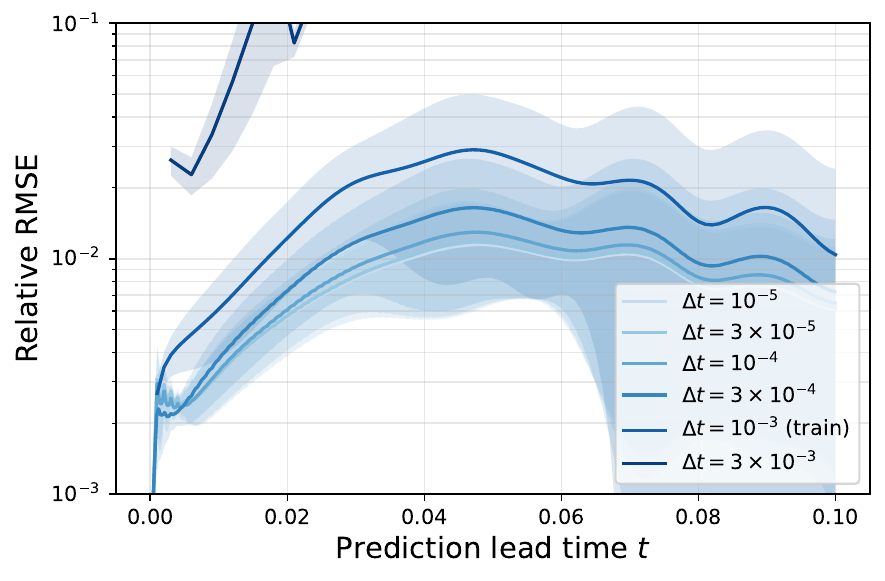}
    \subcaption{prediction error along rollout (1D)}
    \end{minipage}%
    \hfill
    \begin{minipage}{.48\textwidth}
    \centering
    \includegraphics[width=\linewidth]{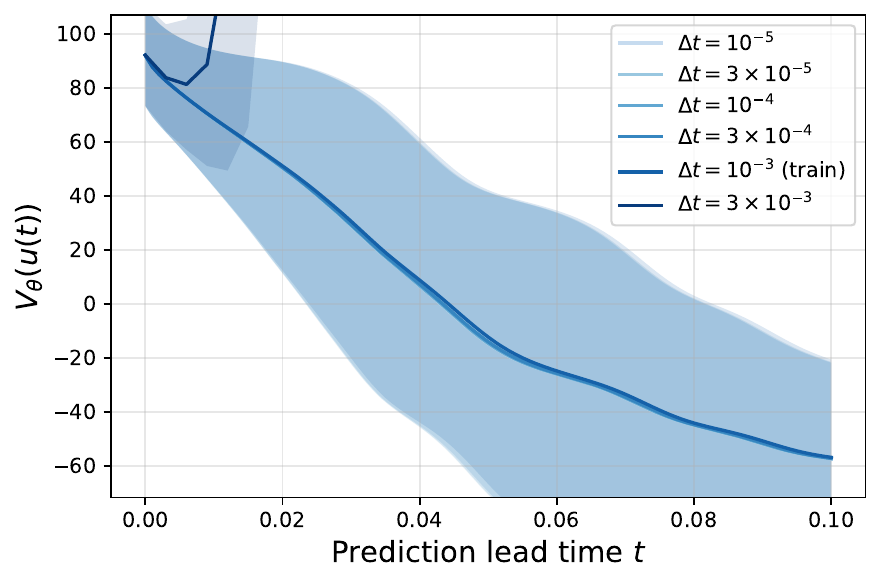}
    \subcaption{learned potential along rollout (1D)}
    \end{minipage}
    \begin{minipage}{.48\textwidth}
    \centering
    \includegraphics[width=\linewidth]{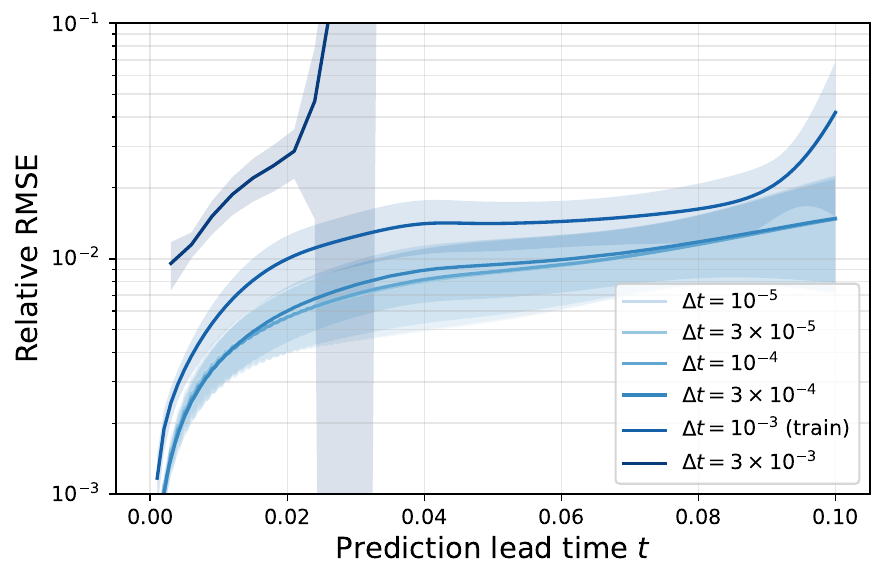}
    \subcaption{prediction error along rollout (2D)}
    \end{minipage}%
    \hfill
    \begin{minipage}{.48\textwidth}
    \centering
    \includegraphics[width=\linewidth]{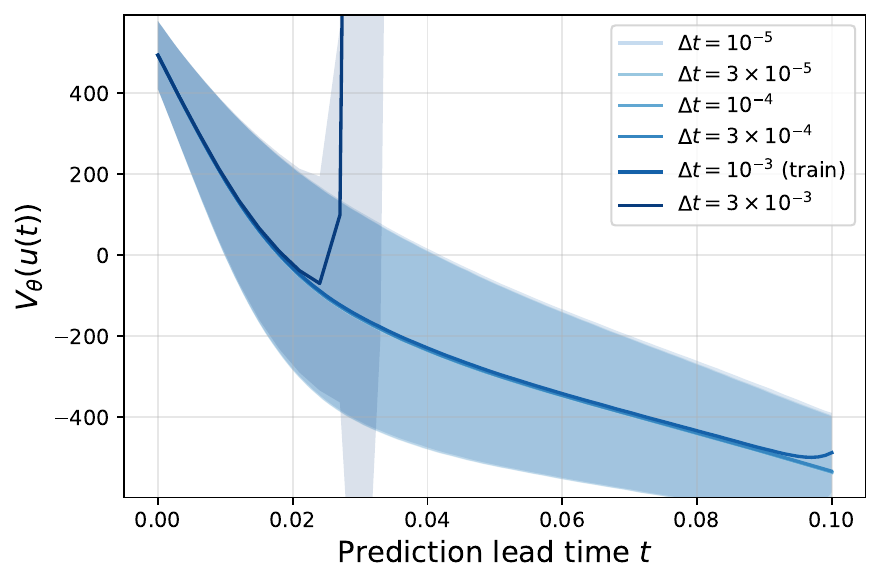}
    \subcaption{learned potential along rollout (2D)}
    \end{minipage}
    \caption{Inference step size scaling for the Allen--Cahn PDE models.}
    \label{fig:dissipation-dt-scaling}
\end{figure}

\subsection{Learning mesoscopic dynamics from microscopic data}
Whereas the continuum examples validate the framework against known reference dynamics, we next turn to a more practical multiscale setting: inferring effective mesoscopic spatio-temporal dynamics directly from microscopic simulation data without assuming a closed-form PDE \textit{a priori}.
In particular, we demonstrate our approach on nonlinear particle chains, a classical setting in which the mesoscopic energy structure reveals how local interactions give rise to energy transfer, nonlinear wave propagation, recurrence, and continuum-scale behavior.
The Fermi--Pasta--Ulam--Tsingou (FPUT) chain \cite{fermi1955studies} represents the classical setting in which a weak nonlinear perturbation of a harmonic interaction, under the long-wave and small-amplitude scaling, produces the KdV-type mesoscopic potential that balance nonlinearity and dispersion \cite{zabusky1965interaction}. 
The polymer chain with Finitely Extensible Nonlinear Elastic (FENE) \cite{warner1972kinetic}, by contrast, introduces a finite-extensibility constraint and the corresponding mesoscopic description is not available. Existing data-driven methods may reproduce coarse trajectories, but they do not identify the key mesoscopic potential that determines whether coarse-graining preserves or reshapes the KdV-type potential.
In this subsection, we show that our method recovers this learned mesoscopic potential as a physically interpretable quantity, enabling a direct comparison between the established FPUT energy mechanism and the finite-extensibility-induced effective potential of the polymer chain.

\subsubsection{Microscopic chain models}
We consider a one-dimensional nearest-neighbor chain with displacement \(q_n\)
and strain \(r_n=q_{n+1}-q_n\). For an interaction potential \(V\), the
microscopic dynamics are
\begin{equation}
    \ddot q_n=\nabla V(r_n)-\nabla V(r_{n-1}),\quad\text{yielding}\quad
    \ddot r_n=\nabla V(r_{n+1})-2\nabla V(r_n)+\nabla V(r_{n-1}).
\end{equation}
For the FPUT chain, $V(r) = c^2r^2/2+\alpha r^3/3$. In this case, the mesoscopic limit is analytically understood. Under the long-wave, small-amplitude scaling
\begin{equation}\label{eq: coarse-graining}
    r_n(t)=\varepsilon^2u(\xi,\tau),\qquad
    \xi=\varepsilon(n-ct),\qquad \tau=\varepsilon^3t,
\end{equation}
Taylor expansion and dominant balance give
\begin{equation}
    -2c u_{\xi\tau}
    =
    \frac{c^2}{12}u_{\xi\xi\xi\xi}
    +
    \alpha (u^2)_{\xi\xi}.
\end{equation}
With periodic or decaying boundary conditions, integration in \(\xi\) yields the KdV equation
\begin{equation}
    u_t+a_{\rm kdv}uu_x+b_{\rm kdv}u_{xxx}=0,
\end{equation}
where \(a_{\rm kdv}=\alpha/c\) and \(b_{\rm kdv}=c/24\). Figure~\ref{fig:chain-coarse-graining} illustrates how microscopic particle-chain models are coarse-grained into mesoscopic variables,

As a contrasting microscopic model, we use the FENE potential and force law
\begin{equation}
    V(r)=-\frac{HR^2}{2}\log\!\left(1-\frac{r^2}{R^2}\right),
    \qquad
    F(r)=\nabla V(r)=\frac{Hr}{1-(r/R)^2}.
\end{equation}
Here \(H\) and \(R\) are fixed. Unlike the polynomial FPUT force, this finite-extensibility law cannot be faithfully approximated by the Taylor expansion and therefore does not reduce to a simple standard continuum model under the same coarse-graining procedure \cref{eq: coarse-graining}, making the FENE chain a more demanding test case for learning mesoscopic dynamics directly from microscopic simulations.

\begin{figure}
    \centering
    \includegraphics[width=\linewidth]{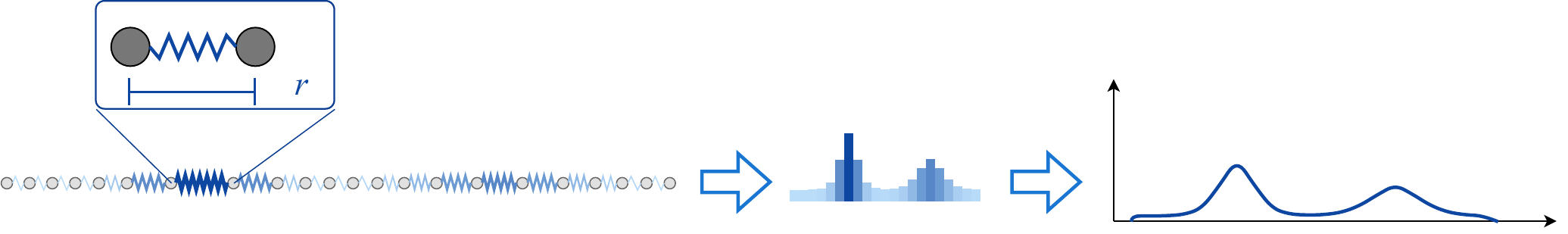}
    \caption{Illustration of microscopic chain models and their coarse-graining.}
    \label{fig:chain-coarse-graining}
\end{figure}

% \subsubsection{Still capable of learning the unknown dynamics}
First, we evaluate the prediction accuracy of our model compared with other data-driven approaches, and the quantitative performance is exhibited in \cref{tab:benchmark}. The trends of superior performance of {\ournetabbrv} on the microscopic datasets mirror those of the continuum examples. Our model consistently yields the highest overall accuracy and stability among all tested methodologies, which provide a reliable learned dynamics for our further study on the interpretability.

\subsubsection{Empirical validation of a priori theoretical analysis}
Recall that in \cref{sec:theoretical-analysis} we have provided a priori theoretical analysis, which applies to any dynamics in our hypothesis class.
In particular, we validate the discrete stability described in \cref{thm:discrete-stability}, which states that the one-step variation for the learned potential should be of order $\mathcal O(\Delta t^2)$. To this end, we plot the variation under different time-step sizes (from $0.1\times$ to $1.0\times$ the step size used for training) for both the two microscopic models in \cref{fig:one-step-potential-evolution}. We compute the median of the variation for robustness since for some rare cases, the predicted trajectory diverges because the network has never been trained with different step sizes. Empirical scaling indicates a variation of approximately $\mathcal O(\Delta t^{1.15})$. While this deviates from the theoretical quadratic law, the discrepancy is physically reasonable. By the one-step variation of the discrete energy \cref{eq:Descent-Lemma},
if the network perfectly identifies the conservative dynamics, then only the quadratic term remains. However, in practice, the learned operators are approximations, resulting in a small but non-zero residual dissipation. The observed scaling of $1.15$ thus reflects a superposition of a small $\mathcal{O}(\Delta t)$ contribution from the fitting error dominating the high-order term $\mathcal{O}(\Delta t^2)$.
\begin{figure}
    \centering
    \begin{minipage}{.48\textwidth}
    \centering
    \includegraphics[width=\linewidth]{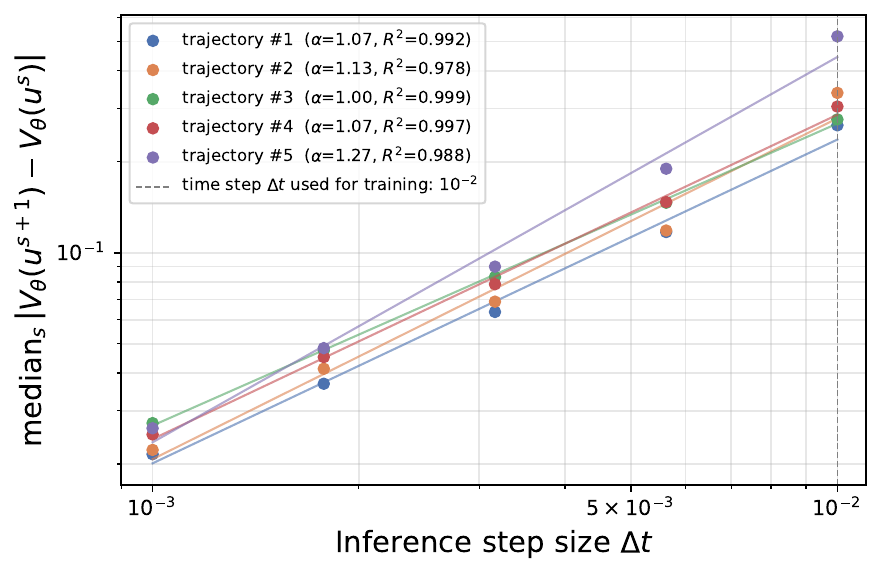}
    \subcaption{FPUT chain}
    \end{minipage}%
    \hfill
    \begin{minipage}{.48\textwidth}
    \centering
    \includegraphics[width=\linewidth]{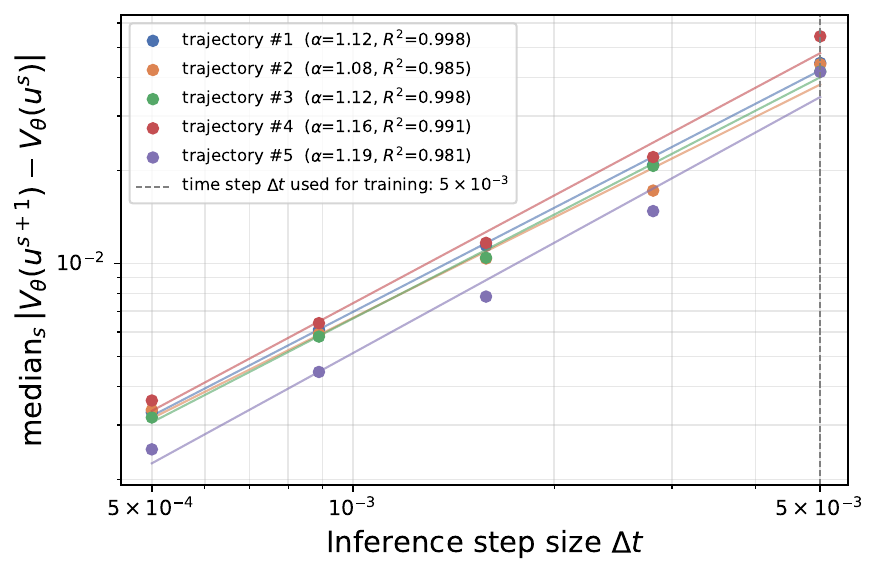}
    \subcaption{FENE chain}
    \end{minipage}
    \caption{One-step variation of the learned potentials $V_\theta$ for the microscopic chain models. For each test trajectory we run a log-log linear fit; the resulting scaling exponent $\alpha$ and the the coefficients of determination $R^2$ are appended in the legend.}
    \label{fig:one-step-potential-evolution}
\end{figure}

\subsubsection{Exploring the potential evolution}
As shown in \cref{fig:ckdv-fene-potential-evolution}, the learned dynamics approximately conserve both the learned potential. These results indicate that the proposed framework can recover the conservative property of the underlying mesoscopic dynamics from microscopic trajectory data alone, without assuming any prior knowledge of the mesoscopic evolution equation.

% As shown in \cref{fig:ckdv-fene-potential-evolution}, the learned dynamics approximately conserve the learned potential and the microscopic chain Hamiltonian. These results underscore the broad applicability of the current framework: it reliably discovers physically consistent and stable macroscopic governing laws regardless of whether an analytical effective equation is known \textit{a priori} or the dynamics must be entirely distilled from raw microscopic trajectories.\zhuoyuan{$\leftarrow$ Instead, in this part we should say, our framework is capable of recovering the conservative property without any knowledge of the mesoscopic dynamics. consider rewriting}
\begin{figure}
    \centering
    \begin{minipage}{.48\textwidth}
    \centering
    \includegraphics[width=\linewidth]{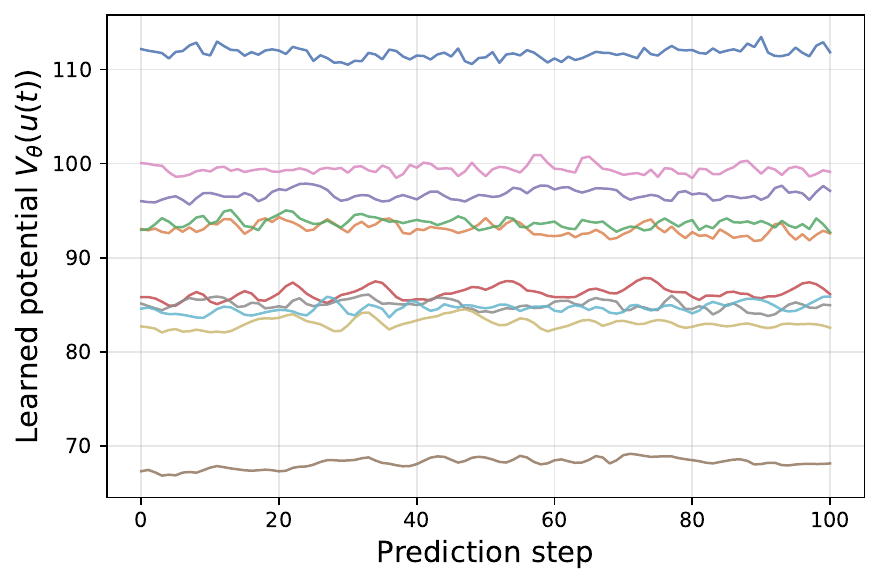}
    \subcaption{FPUT chain}
    \end{minipage}%
    \hfill
    \begin{minipage}{.48\textwidth}
    \centering
    \includegraphics[width=\linewidth]{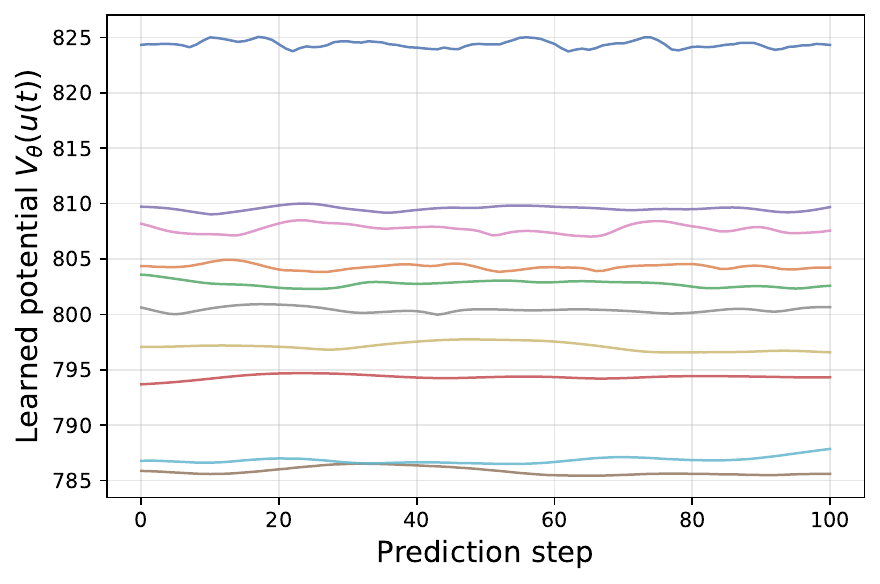}
    \subcaption{FENE chain}
    \end{minipage}
    \caption{Evolution of the learned potentials $V_\theta$ for the microscopic chain models. We use colors to distinguish different trajectories.
    }
    \label{fig:ckdv-fene-potential-evolution}
\end{figure}

\subsubsection{Interpretability of the learned potentials}
By revisiting our ansatz of $V$ defined in \cref{eq:V-ansatz}, intuitively we may want to know the learned component of the potential, or more specifically, the pointwise $F(u)$ term and the $V_2(u)$ designed for the residuals.

Firstly, we evaluate the learned components $V_0$ and $V_2$ on a randomly sampled initial state $u_0$ scaled by a varying amplitude $a$ in \cref{fig:V0-V2-a}. Two critical observations emerge from this scaling. For both cases, $V_0(a u_0)$ carries the dominant amplitude dependence, whereas $V_2(a u_0)$ varies on a much smaller scale. Thus $V_2$ represents a weaker correction beyond the pointwise constitutive ansatz, capturing residual effects such as nonlocal or higher-order interactions. Besides, the form of $V_0$ is strongly dependent on the microscopic physics. For the FPUT chain, $V_0(a u_0)$ displays a polynomial-type growth, consistent with the KdV-type Hamiltonian mechanism in which the local energy generates nonlinear transport. In contrast, the learned pointwise energy $V_0$ for the FENE chain initially follows a similar behavior for small amplitudes but then departs markedly from the FPUT profile, exhibiting substantially stronger, super-polynomial growth.
This contrast shows that the mesoscopic dynamics learned from the FENE chain is governed by a different nonlinear response once the finite-extensibility effect becomes active.

\begin{figure}
    \centering
    \includegraphics[width=0.95\linewidth]{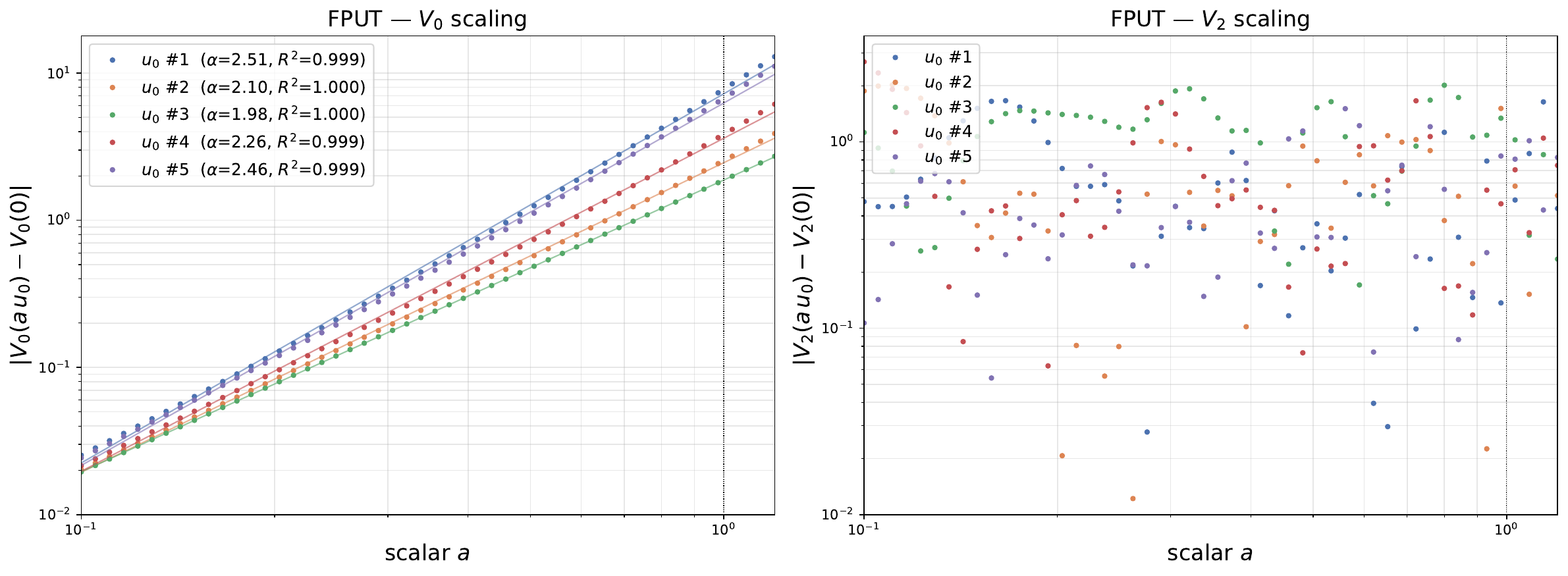}
    \includegraphics[width=0.95\linewidth]{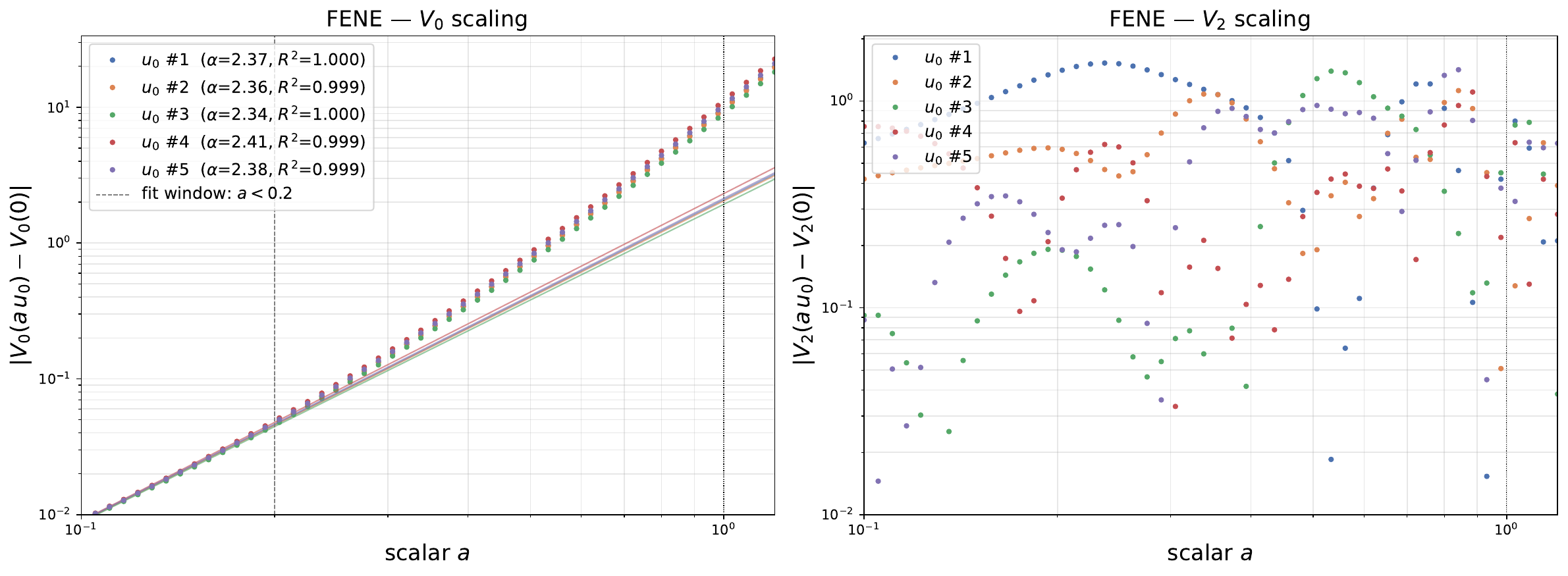}
    \caption{Scaling property of the components $V_0$ and $V_2$ for the microscopic chain models. For each test $u_0$, we run a log-log linear fit of $V_0$ w.r.t. the scalar $a$; the resulting scaling exponent $\alpha$ and the the coefficients of determination $R^2$ are appended in the legend.}
    \label{fig:V0-V2-a}
\end{figure}

% \zhuoyuan{
% There are two points we need to emphasize:
% \begin{itemize}
%     \item The scale of the residual $V_2(u)$ is relatively slow when compared with $V_0(u)$, which indicates that for both the two learned mesoscopic dynamics, apart from the pointwise potentials and local interactions, we require some other residual terms (e.g., global interactions in the VPFP limit\needcite)
%     \item When comparing the learned pointwise $V_0(u)$, for the FPU chain, the learned local energy displays polynomial-type growth, again consistent with a KdV-type nonlinear Hamiltonian mechanism. By contrast, the learned local energy for the FENE chain shows substantially stronger, super-polynomial, growth.
% \end{itemize}
% }

We also plot the learned $F(u)$ term in \cref{fig:Fu-varying} by varying the variable $u$ within a range where most of the training data stay to avoid out-of-distribution testing.
\begin{figure}
    \centering
    \begin{minipage}{.48\textwidth}
    \centering
    \includegraphics[width=\linewidth]{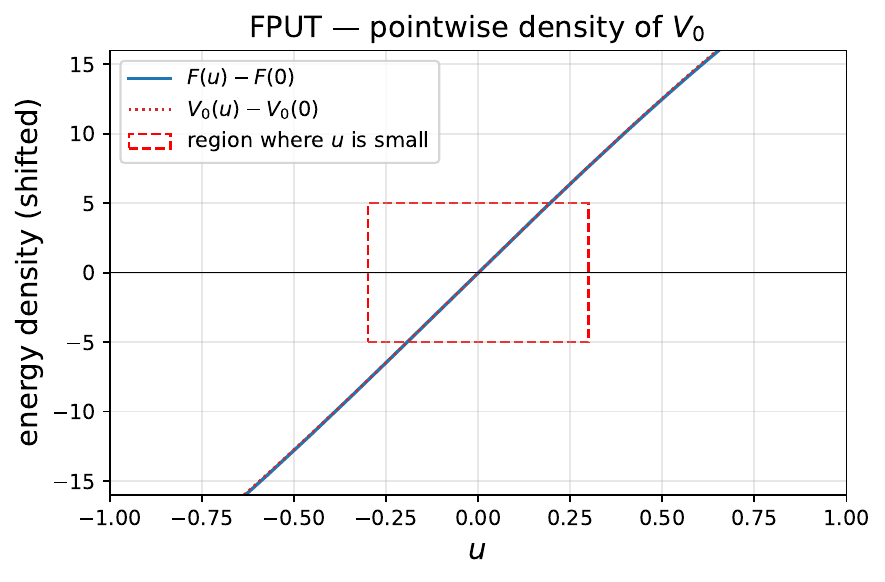}
    % \subcaption{FPU chain}
    \end{minipage}%
    \hfill
    \begin{minipage}{.48\textwidth}
    \centering
    \includegraphics[width=\linewidth]{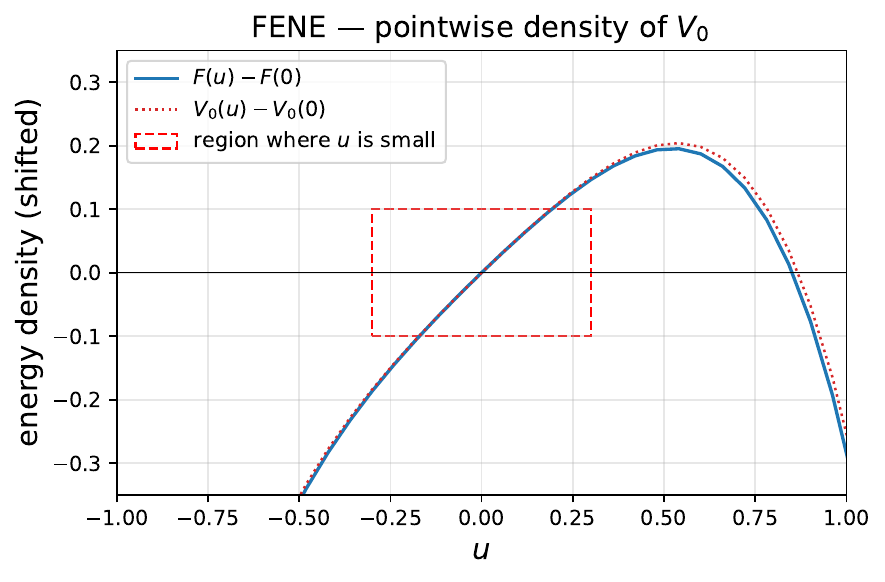}
    % \subcaption{FENE chain}
    \end{minipage}
    \caption{The pointwise component $F(u)$ and $V_0(u)$ of the learned potentials $V_\theta$}
    \label{fig:Fu-varying}
\end{figure}
The learned profiles further confirm the distinction between the two microscopic chains. For the FPUT chain, $F(u)$ is predominantly asymmetric, nearly odd over the sampled range, and displays super-quadratic growth, in qualitative agreement with the form of a KdV-type local Hamiltonian density. In contrast, the learned profile for the FENE chain is similar to the FPUT case only near the origin, but departs substantially as the amplitude increases. Taken together, these observations suggest that the two microscopic interaction laws induce different effective nonlinear constitutive relations at the mesoscopic level. In particular, the learned FENE dynamics are not well described by a KdV-type equation with merely modified coefficients, but instead encode a distinct effective nonlinear constitutive relation at the mesoscopic scale.
\section{Conclusion}\label{sec:conclusion}
In this work, we propose a departure from standard pipelines of modelling dynamic processes via evolution equations. Rather than deriving a fixed, instance-wise effective equation and subsequently analyzing its properties, we introduce a framework centered on a highly expressive yet mathematically constrained hypothesis class. By lifting the generalized Onsager principle to infinite-dimensional Gelfand triples, we define a unified class encompassing both dissipative and conservative mesoscopic dynamics. Crucially, we establish theoretical guarantees uniformly across this entire class. This ensures that foundational physical properties, including global well-posedness, asymptotic stability, unique factorization identifiability, and discrete energy dissipation, are guaranteed \textit{a priori}. Consequently, regardless of the specific data-driven model used to fit the dynamics, the resulting system inherently preserves these structural and analytical certainties.
% {\ournet} serves as a concrete, end-to-end realization of this philosophy. 
% By lifting the generalized Onsager principle to infinite-dimensional Gelfand triples, we defined a hypothesis class broad enough to encompass both dissipative and conservative mesoscopic dynamics. Crucially, we demonstrate that the mathematically stringent constraints of operator symmetry, positive semi-definiteness, and skew-symmetry can be exactly resolved via pointwise algebraic conditions in the spectral domain. This spectral diagonalization not only makes the framework computationally tractable, but it also allows us to prove global well-posedness, asymptotic stability, unique factorization identifiability, and discrete energy dissipation for the learned models. These guarantees usually absent in other data-driven methods.

% Empirically, this structured approach proved highly effective. Whether learning from canonical continuum PDEs (like the KdV and Allen--Cahn equations) or extracting effective macroscopic descriptions directly from atomistic particle chains (like the FPU and FENE models), S-OnsagerNet consistently outperformed unconstrained baselines in predictive accuracy and long-term numerical stability. It faithfully recovered the correct underlying physical regimes, conserving Hamiltonian energy where appropriate and recovering the exact affine transformations of dissipative free-energy functionals, thereby confirming that the learned operators represent genuine physical laws. 

Empirically, this structured approach has shown effectiveness across both continuum and particle-based settings with a specific implementation, \ournet. For continuum PDEs such as the KdV and Allen--Cahn equations, {\ournet} consistently improves predictive accuracy and long-time numerical stability over unconstrained baselines. Moreover, it faithfully recovers the correct underlying physical regimes, explicitly capturing the exact affine transformations of dissipative free-energy functionals and the conservation of purely Hamiltonian potentials. 
Furthermore, when learning mesoscopic dynamics directly from microscopic particle chains, the formulation of {\ournet} provides vital interpretable diagnostics. By analyzing the amplitude dependence of the learned energy components, we directly observe how microscopic interaction laws are reflected in the effective nonlinear constitutive relation at the mesoscopic scale. This interpretability allows the framework to distinguish KdV-type Hamiltonian behavior from mesoscopic dynamics governed by finite extensibility in the FENE chain.

While the proposed approach is a first step towards data-driven construction of mesoscopic descriptions, important theoretical and computational challenges remain. The current reliance on convolution structures to diagonalize the operators $\mM$ and $\mW$ restricts the methodology to periodic boundary conditions. Extending this uniform class-level analysis to handle complex spatial geometries, diverse boundary conditions, and fully state-dependent transport operators remains a critical next step.
More fundamentally, while the proposed hypothesis class is highly expressive, it does not universally encompass all non-equilibrium mesoscopic phenomena. For complex multiscale systems governed by fundamentally different thermodynamic or kinetic rules, it will be necessary to propose and integrate alternative structure-preserving physical principles within this hypothesis-driven framework.
Ultimately, we envision this approach as a rough blueprint for the automated discovery of reliable, physically consistent, and analytically sound multiscale models across a diverse spectrum of governing laws.

% \clearpage

% \input{app/chain_v2}
% \input{app/dataset}
% \input{app/network}
% \input{app/ablation}

% \input{app/misc}
% \input{app/misc-v2}
% \input{app/misc-v3}

\section*{Data availability}
All the codes for data generation, network training, and visualization are publicly available on the GitHub repository: \texttt{\url{https://github.com/MLDS-NUS/Meso-SpecOnsNet}}.

\section*{Acknowledgments}
% This project is supported by the National Research Foundation, Singapore, under its AI Singapore Programme (AISG Award No.: AISG3-RP-2022-028).
% Q.L. and A.Z. are partially supported by the National Research Foundation, Singapore, under the NRF fellowship (Project No. NRF-NRFF13-2021-0005). 
The computational work for this article was partially performed on resources of the National Supercomputing Centre, Singapore (\url{https://www.nscc.sg}).
During the preparation of this work, the authors used large language models (LLMs) to assist with the code writing and to polish the written text for spelling and grammar. The authors thoroughly reviewed and edited the content and take full responsibility for the final publication.
% During the preparation of this work, the authors used ChatGPT (5.4) and Gemini (3.1 Pro) to polish the written text for spelling and grammar. The code writing is assisted with Github Copilot and Claude Code (Opus 4.6). The authors reviewed and edited the content as needed and take full responsibility for the content of the publication.

\appendix

\section{Experiment details}

\subsection{Baselines}\label{sec:baselines}
% We compare our model with the following baselines.
% \begin{itemize}
%     \item Classical solver: a one-step classical scheme accessible to the reference dynamics with a larger step size;
%     \item Fourier Neural Operator (FNO): learning the dynamics with a residual structure in the physical space;
%     \item OnsagerNet: with a fixed spatial discretization, we treat the system as a high-dimensional ODE;
%     \item \ournet: we apply the parameterization as described in \cref{eq:parameterized-update,eq:parameterization-mu}. The detailed workflow is displayed in \cref{fig:architecture}.
% \end{itemize}
\paragraph{Classical solver with true operators}We use the same equation as that of the reference trajectory to evolve the solution numerically. A larger step size is set to keep consistent with the data-driven approaches. Meanwhile,  we adopt the SBDF1 scheme (see \cref{sec:numerical-stepper}) as a one-step variant of the reference scheme for a fair comparison. This model serves as a baseline for one-step classical methods with a large step size.
\paragraph{Fourier Neural Operator (FNO)} We use the FNO architecture \cite{li2020fourier} to directly learn the dynamics in the physical space, which serves as a robust baseline in the field of operator learning. More specifically, we set
\begin{equation}
    u^{s+1}\approx u^s+\gamma_\theta(u^s)\Delta t,
\end{equation}
where the network $\gamma_\theta$ is implemented with the FNO architecture.
\paragraph{OnsagerNet} We use the OnsagerNet architecture \cite{Yu2021OnsagerNet} originally designed for learning Onsager principle. Since the reference snapshots are saved as grid data, we can restate the learning task in the discrete space. Concretely, we treat the $N_x$-point PDE grid directly as an $N_x$-dimensional ODE, evolving under
\begin{equation}
    \ddt{\mathbf{u}} = -\bigl(M_\theta(\mathbf{u}) + W_\theta(\mathbf{u})\bigr)\nabla V_\theta(\mathbf{u}),\quad\mathbf u\in\mR^{N_x}.
\end{equation}
The matrix $M(\mathbf u)$ and $W(\mathbf u)$ are parameterized as low-rank matrices for computational efficiency, and the potential is implemented as an MLP.
\paragraph{{\ournet} (\ournetabbrv, our work)}We apply the parameterization as described in \cref{eq:parameterized-update,eq:parameterization-mu}. The detailed workflow is displayed in \cref{fig:architecture}.

\subsection{Evaluation metric}\label{sec:evaluation-metric}
To quantify the prediction performance, we introduce the relative RMSE at time $t$, evaluated as
\begin{equation}
    \textrm{Relative RMSE}(t)=\sqrt{\frac{\sum_i\norm{u_\mathrm{pred}(t)(x_i)-u_\mathrm{ref}(t)(x_i)}_2^2}{\sum_i\norm{u_\mathrm{ref}(t)(x_i)}_2^2}},
\end{equation}
where $u_\mathrm{ref}(t)$ and $u_\mathrm{pred}(t)$ are reference and prediction states for time $t$, respectively. 
\subsection{Training cost}
The network training is performed on a single NVIDIA A100 GPU based on the PyTorch framework, and each experimental configuration requires about 2 hours. Details of the optimizer and learning-rate scheduler for each experiment are given in \cref{sec:network-and-training}. 
\section{Examples}\label{sec:examples}
In this section, we exhibit some classical PDE dynamics covered by the generalized Onsager principle, or more specifically, those satisfying \cref{ass:symmetry}, including pure gradient flows ($\mW\equiv 0$), purely Hamiltonian flows ($\mM\equiv 0$), and mixed dynamics such as the incompressible Navier--Stokes equation, where both $\mM$ and $\mW$ do not vanish.
\begin{Exp}[(Hyper/Fractional) diffusion]
    Let the state space be $\mathcal{X} = L^2(\Omega)$. Consider the quadratic energy $V: \mathcal{X} \to \mathbb{R}$:
    \begin{equation}
        V(u) = \int_\Omega \frac{1}{2} u^2 \,\mathrm{d}x, \quad \frac{\delta V}{\delta u} = u.
    \end{equation}
    This yields the \underline{(hyper/fractional) diffusion} dynamics:
    \begin{equation}
        \partial_t u = -(-\Delta)^s u = -\big[ \underbrace{(-\Delta)^s}_{\mathcal{M}(u)} + \underbrace{0}_{\mathcal{W}(u)} \big] \underbrace{u}_{\frac{\delta V}{\delta u}}.
    \end{equation}
    Note that for $s=1$, this recovers the classical Heat equation.
\end{Exp}

\begin{Exp}[Reaction--Diffusion]\label{ex:reaction-diffusion}
    Let the state space be $\mathcal{X} = H^1(\Omega)$. Consider the Ginzburg--Landau free energy $V: \mathcal{X} \to \mathbb{R}$:
    \begin{equation}
        V(u) = \int_\Omega \left[ \frac{\varepsilon^2}{2}|\nabla u|^2 + F(u) \right] \mathrm{d}x, \quad \frac{\delta V}{\delta u} = -\varepsilon^2 \Delta u + F'(u),
    \end{equation}
    where $F(u)$ is a double-well potential. This energy generates two distinct dynamics depending on the choice of operators:
    \begin{enumerate}
        \item \underline{Allen--Cahn (Non-conserved):} By choosing $\mathcal{M}=I$, we obtain the $L^2$-gradient flow:
        \begin{equation}
            \partial_t u = \Delta u - F'(u) = -\big( \underbrace{I}_{\mathcal{M}(u)} + \underbrace{0}_{\mathcal{W}(u)} \big) \Big[ \underbrace{-\Delta u + F'(u)}_{\frac{\delta V}{\delta u}} \Big].
        \end{equation}

        \item \underline{Cahn--Hilliard (Conserved):} By choosing $\mathcal{M}=-\Delta$, we obtain the $H^{-1}$-gradient flow:
        \begin{equation}
            \partial_t u = \Delta(-\varepsilon^2 \Delta u + F'(u)) = -\big( \underbrace{-\Delta}_{\mathcal{M}(u)} + \underbrace{0}_{\mathcal{W}(u)} \big) \Big[ \underbrace{-\varepsilon^2 \Delta u + F'(u)}_{\frac{\delta V}{\delta u}} \Big].
        \end{equation}
    \end{enumerate}
\end{Exp}
% \begin{Exp}[(hyper/fractional) diffusion]
% Consider the energy $V:\mX\to\mR$ defined as
% \begin{equation}
%     V(u)=\int_\Omega\frac{u^2}2\md x,\quad \frac{\delta V}{\delta u}=u,
% \end{equation}
% which leads to the \underline{(hyper/fractional) diffusion} dynamics:
%     \begin{equation}
%         \partial_tu=-(-\Delta)^s u=-[\underbrace{(-\Delta)^s}_{\mM(u)}+\underbrace{0}_{\mW(u)}]\underbrace{u}_{\frac{\delta V}{\delta u}}.
%     \end{equation}
% \end{Exp}
% \begin{Exp}[reaction--diffusion]\label{ex:heat}
% Consider the energy $V:\mX\to\mR$ defined as
% \begin{equation}
%     V(u)=\int_\Omega\left[\frac{\varepsilon^2}2|\nabla u|^2+F(u)\right]\md x,\quad\frac{\delta V}{\delta u}=-\Delta u+F'(u),
% \end{equation}
% which leads to
% \begin{enumerate}
%     \item \underline{Allen-Cahn} dynamics:
%     \begin{equation}
%         \partial_tu=\Delta u-F'(u)=-(\underbrace{I}_{\mM(u)}+\underbrace{0}_{\mW(u)})[\underbrace{-\Delta u+F'(u)}_{\frac{\delta V}{\delta u}}],
%     \end{equation}
%     where we set $\varepsilon=0$, and
%     \item \underline{Cahn-Hilliard} dynamics:
%     \begin{equation}
%         \partial_tu=\Delta(-\varepsilon^2\Delta u+F'(u))=-(\underbrace{-\Delta}_{\mM(u)}+\underbrace{0}_{\mW(u)})[\underbrace{-\varepsilon\Delta u+F'(u)}_{\frac{\delta V}{\delta u}}].
%     \end{equation}
% \end{enumerate}
% \end{Exp}
% \begin{Exp}[Non-local Cahn-Hilliard]
%     \todo (Cahn-Hilliard with non-local free genegy given by a convolution)
% \end{Exp}

\begin{Exp}[KdV]
    Let the spatial dimension $d=1$. Consider the energy $V:\mX\to\mR$ defined as
\begin{equation}
    V(u) = \int_\mR\left( \frac12 u_x^2 - u^3\right)\md x,\quad\frac{\delta V}{\delta u}=-u_{xx}-3u^2.
\end{equation}
We have the \underline{KdV} dynamics:
\begin{equation}
    \partial_tu=-u_{xxx}-6uu_x=-[\underbrace{0}_{\mM(u)}+\underbrace{(-\partial_x)}_{\mW(u)}](\underbrace{-u_{xx}-3u^2}_{\frac{\delta V}{\delta u}}).
\end{equation}
\end{Exp}

\begin{Exp}[nonlinear Schr\"odinger]
Let $u=(p,q)\in L^2(\Omega;\mR^2)$ and write $\psi = p+\mi q$. By setting
\begin{equation}
    \mM(u)\equiv0,\quad\mW(u)\equiv \begin{pmatrix}0&-1\\1&0\end{pmatrix},
\end{equation}
and
\begin{equation}
    V(u) = \int_\Omega \left( \frac12|\nabla u|^2 + \frac{\lambda}{4}|u|^4\right)\md x,
\end{equation}
we have the \underline{nonlinear Schr\"odinger} dynamics:
\begin{equation}
    \mi\psi_t = -\Delta\psi + \lambda|\psi|^2\psi.
\end{equation}
\end{Exp}

\begin{Exp}[1D conservation laws]
    Let the spatial dimension $d=1$. Consider the energy $V:\mX\to\mR$ defined as
\begin{equation}
    V(u) = \int_\mR G(u)\md x,\quad\frac{\delta V}{\delta u}=G'(u).
\end{equation}
We have the \underline{1D conservation} dynamics:
\begin{equation}
    \partial_tu=-G'(x)_x=-(\underbrace{0}_{\mM(u)}+\underbrace{\partial_x}_{\mW(u)})\underbrace{G'(u)}_{\frac{\delta V}{\delta u}}.
\end{equation}
\end{Exp}

\begin{Exp}[Incompressible Navier--Stokes]\label{ex:navier-stokes}
Let $\Omega\subset\mathbb R^d$ and $\mathcal X$ be the $L^2$-closure of smooth divergence-free vector fields with suitable boundary conditions. Let
\begin{equation}
    V(u) = \frac12\int_\Omega |u|^2\md x,\quad\frac{\delta V}{\delta u}=u\in \mX.
\end{equation}
Let $P$ be the Leray projector onto divergence-free fields, and define
\begin{equation}
    \mM(u)v = \nu P\Delta v,\quad \mW(u)v = -P\left[(u\cdot\nabla) v\right],
\end{equation}
then we have the \underline{incompressible Navier--Stokes} dynamics:
\begin{equation}
    \partial_t u + (u\cdot\nabla)u = -\nabla p + \nu\Delta u, \quad \nabla\cdot u=0
\end{equation}
with the pressure enforced by $P$.
\end{Exp}
\begin{table}[h]
    \centering
    \resizebox{\textwidth}{!}{%
    \begin{tabular}{c|c|c|c|c}
    \toprule
         Dynamics & $V(u)=\int_\Omega\colorbox{lightgray}{$\cdot$}~\md x$ & $\mM$ & $\mW$ & differential equation \\
    \midrule
         diffusion&$\frac12u^2$&$(-\Delta)^s$&0&$u_t+(-\Delta)^su=0$\\
         Allen-Cahn & $\frac12|\nabla u|^2+F(u)$ & $I$ & $0$ & $\partial_tu=\Delta u-F'(u)$ \\
         Cahn-Hilliard & $\frac{\varepsilon^2}2|\nabla u|^2+F(u)$ & $-\Delta$ & $0$ & $\partial_tu=\Delta [-\varepsilon^2u+F'(u)]$ \\
         \midrule
         KdV & $\frac12u_x^2-u^3$ & $0$ & $-\partial_x$ & $u_t+u_{xxx}+6uu_x=0$\\
         Schr\"odinger & $\frac12|\nabla u|^2+\frac\lambda4|u|^4$ & $0$ & $\mi$ & $\mi\psi_t=-\Delta\psi+\lambda|\psi|^2\psi$\\
         1D conservation&$G(u)$&0&$\partial_x$&$u_t+G'(u)_x=0$\\
         \midrule
         Navier--Stokes & $\frac12|u|^2$ & $\nu P\Delta$ & $-P[(u\cdot\nabla)\cdot]$ & $u_t+(u\cdot\nabla)u=-\nabla p+\nu\Delta u$\\
    \bottomrule
    \end{tabular}%
    }
    \caption{Some classical PDE dynamics that fit our setup}
    \label{tab:dynamics}
\end{table}

\section{Detailed proofs}\label{sm:proofs}
\subsection{Proof of \cref{thm:global-existence}}\label{sec:proof-of-global-existence}
Let $\mP_n$ be the canonical projection from $\mH=L^2(\Omega)$ to $\mH_n$ defined as
\begin{equation}
    \mP_n:u\mapsto\sum_{|k|\le n}\Hinner{u}{e_k} e_k,
\end{equation}
where we set $\mH_n$ as the finite-dimensional space spanned by the Fourier modes $e_k$ for $|k|\le n$. Consider the time-dependent ODE system
\begin{equation}
    \partial_t u_n(t)=\mP_n\mG(u_n(t)),\quad u_n(0)=\mP_nu_\mathrm{init},
\end{equation}
which admits a local solution $u_n$ on some interval $[0, T_n)$ by the Peano existence theorem since the subspace $\mathcal{H}_n$ is finite-dimensional and the right-hand side is continuous.

Next, we show that the functional $V$ decays for any solution $u_n(t)$. For simplicity, we suppress the dependency on $u$ in the operators. For each $n$, we have
\begin{equation}\label{eq:Pn-Mmu}
        \mP_n(\mM\mu)(z)=\sum_{|k|\le n}\Hinner{\mM\mu}{e_k}e_k(z)=\widehat\mM^{[k]}\Xcouple{\mu}{\overline{e_k}}e_k(z)
\end{equation}
by \cref{thm:spectral-representation}.
Meanwhile, defining $\mP_n^*:\mX^*\to\mX^*$ as the adjoint of $\mP_n$ via
\begin{equation}
    \Xcouple{\mP_n^*\mu}{v}=\Xcouple{\mu}{\mP_n v},
\end{equation}
we have
\begin{equation}
    \mM\left(\mP_n^*\mu\right)(z)=\Xcouple{\mP_n^*\mu}{K_\mM(z-\cdot)}=\Xcouple{\mu}{\sum_{|k|\le n} \Hinner{K_\mM(z-\cdot)}{e_k} e_k},
\end{equation}
where
\begin{equation}
\begin{aligned}
    \Hinner{K_\mM(z-\cdot)}{e_k}&=\int_\Omega K_\mM(z-y)\overline{e_k(y)}\md y\\
    &=\overline{e_{k}(z)}\int_\Omega K_\mM(z-y)e_{k}(z-y)\md y=\overline{e_{k}(z)}\widehat\mM^{[-k]}.
\end{aligned}
\end{equation}
Hence
\begin{equation}
        \mM\left(\mP_n^*\mu\right)(z)=\sum_{|k|\le n}\widehat\mM^{[-k]}\Xcouple{\mu}{e_k}\overline{e_k(z)},
\end{equation}
which is identical with \cref{eq:Pn-Mmu} since $e_{-k}(z)=e_k(-z)=\overline{e_k(z)}$. The same argument applies to $\mW$, so
\begin{equation}
    \mP_n\mM\mu=\mM\mP_n^*\mu,\qquad\mP_n\mW\mu=\mW\mP_n^*\mu.
\end{equation}
It follows that
\begin{equation}\label{eq:energy-dissipation-app-sm}
\begin{aligned}
    \ddt{} V(u_n(t)) &= -\Xcouple{\mu(u_n)}{\mP_n[\mM(u_n) + \mW(u_n)] \mu(u_n)} \\
    &= -\Xcouple{\mP_n^*\mu(u_n)}{\mP_n[\mM(u_n) + \mW(u_n)] \mu(u_n)} \\
    &= -\Xcouple{\mP_n^*\mu(u_n)}{[\mM(u_n) + \mW(u_n)] \mP_n^*\mu(u_n)}\le0.
\end{aligned}
\end{equation}
Consequently, the energy dissipation leads to
\begin{equation}
    V(u_n(t))\le V(u_n(0))=V(\mP_n(u_\mathrm{init})),\quad\forall t\in[0,T_n).
\end{equation}
Note that $\mP_n(u_\mathrm{init})\to u_\mathrm{init}$ as $n\to\infty$, so the right-hand side is bounded by a constant independent of $n$ from above. By the coercivity of $V$,
\begin{equation}\label{eq:un-bounded}
    \|u_n(t)\|_\mX\le C_1,\quad\forall n,\quad\forall t\in[0,T_n).
\end{equation}
Hence each solution $u_n$ cannot blow up in finite time, and we may extend it to $[0,+\infty)$. Furthermore,
\begin{equation}\label{eq:ut-bounded}
    \|\partial_t u_n(t)\|_{\mX}\le\|\mG(u_n(t))\|_\mX \le C_2,\quad\forall n,\quad\forall t\in[0,+\infty),
\end{equation}
since $\mG$ is Lipschitz continuous resulting from \cref{ass:lipschitz-v2}.

Till now, we have shown the uniform boundness for $u_n$ and $\partial_tu_n$. According to \cref{ass:X-constraints}, the embedding $\mX \hookrightarrow L^2(\Omega)$ is compact and $L^2(\Omega) \hookrightarrow \mX^*$ is continuous, sothe space
\begin{equation}
    \tilde W(0,T;\mX):=\left\{u\in L^2(0,T;\mX)\mid\partial_t u\in L^2(0,T;\mX^*)\right\}
\end{equation}
is compactly embedded in $L^2(0,T;L^2(\Omega))$ for any fixed $T>0$ by the Aubin--Lions lemma.
Let $u_n^{(0)}=u_n$ and take an arbitrary increasing sequence $T_m\to+\infty$ as $m\to+\infty$. For each $m>0$, we can find a subsequence of $\left\{u_n^{(m-1)}\right\}_n$ with a weak limit in $L^2(0,T_m;\mX)$ by the reflexivity of $L^2(0,T_m;\mX)$ and the boundness according to \cref{eq:un-bounded,eq:ut-bounded}. Recursively we find a subsequence $\left\{u_n^{(m)}\right\}_n\subseteq\left\{u_n^{(m-1)}\right\}_n$ such that
\begin{equation}
\begin{aligned}
    u_n^{(m)}&\rightharpoonup u^{(m)}\in L^2(0,T_m;\mX),\\
    \partial_tu_n^{(m)}&\rightharpoonup v^{(m)}\in L^2(0,T_m;\mX),\\
    u_n^{(m)}&\to u^{(m)}\in L^2(0,T_m;L^2(\Omega)).
\end{aligned}
\end{equation}
By the uniqueness of strong limits, $u^{(m+1)}$ coincides with $u^{(m)}$ on $L^2(0,T_m;L^2(\Omega))$ for any $m>0$, so by induction there exists a unique limit $u$ defined on $t\in[0,+\infty)$ such that $u_n^{(m)}\to u$ in $L^2(0,T_m;L^2(\Omega))$ for any $m>0$.
Meanwhile, for any $m>0$ and any sufficiently smooth test function $\phi$,
\begin{equation}
    \int_0^{T_m}\Xcouple{\phi}{\partial_tu_n^{(m)}}\md t=-\int_0^{T_m}\Xcouple{\partial_t\phi}{u_n^{(m)}}\md t.
\end{equation}
By taking the limit $n\to\infty$, we have that
\begin{equation}
    \int_0^{T_m}\Xcouple{\phi}{v^{(m)}}\md t=-\int_0^{T_m}\Xcouple{\partial_t\phi}{u}\md t,
\end{equation}
which indicates that $\partial_tu=v^{(m)}$ in $L^2(0,T_m;\mX)$. To sum up, without loss of generality, we may assume that there exists $u$ defined on $t\in[0,+\infty)$ such that for any finite time $T>0$,
\begin{equation}
    u_n\rightharpoonup u\in L^2(0,T;\mX),\quad\partial_tu_n\rightharpoonup\partial_tu\in L^2(0,T;\mX),\text{ and } u_n\to u\in L^2(0,T;\mH).
\end{equation}

Finally, we prove that the limit function $u$ satisfies \cref{eq:onsager-pde}, and we need to handle the convergence of the nonlinear operator $\mG(u_n)$. Consider
\begin{equation}
    X_n(w) = \int_0^T \Hinner{\mG(u_n) - \mG(w)}{u_n - w}\md t
\end{equation}
for an arbitrary $w \in L^2(0,T;\mX)$. 
The term $X_n(w)$ can be decomposed as
\begin{equation}
    X_n(w) = \underbrace{\int_0^T \Hinner{\mG(u_n)}{u_n} \md t}_{Y_n} \underbrace{- \int_0^T \Hinner{\mG(u_n)}{w} \md t - \int_0^T \Hinner{\mG(w)}{u_n - w}\md t}_{Z_n(w)}.
\end{equation}
The $Z_n$ term converges as
\begin{equation}
\begin{aligned}
    Z_n(w) &= -\int_0^T \Hinner{\partial_t u_n}{w}\md t - \int_0^T \Hinner{\mG(w)}{u_n - w}\md t\\
    &\to-\int_0^T \Hinner{\partial_t u}{w}\md t - \int_0^T \Hinner{\mG(w)}{u - w}\md t
\end{aligned}
\end{equation}
as $n\to\infty$ due to the weak convergence of $u_n$ and $\partial_tu_n$. Meanwhile, for the $Y_n$ term we have
\begin{equation}
    Y_n = \int_0^T \Hinner{\partial_tu_n}{u_n}\md t = \frac{1}{2}\|u_n(T)\|_2^2 - \frac{1}{2}\|u_n(0)\|_2^2.
\end{equation}
For any test function $\phi\in \mX$,
\begin{equation}
    \Xcouple{\phi}{u_n(T)}-\Xcouple{\phi}{u_n(0)}=\int_0^t\Xcouple{\phi}{\partial_tu_n}\md t,
\end{equation}
which converges to
\begin{equation}
    \int_0^t\Xcouple{\phi}{\partial_tu}\md t=\Xcouple{\phi}{u(T)}-\Xcouple{\phi}{u(0)}
\end{equation}
as $n\to\infty$. Recall that by definition,
\begin{equation}
    u_n(0)=\mP_n(u(0))\to u(0)\quad\text{in }L^2(\Omega),
\end{equation}
so we have
\begin{equation}
    \Xcouple{\phi}{u_n(T)}\to\Xcouple{\phi}{u(T)}
\end{equation}
as $n\to\infty$, and thus $u_n(T)$ converges weakly to $u(T)$ as well. As a result, we have
\begin{equation}
    \begin{aligned}
    \liminf_{n\to\infty}X_n(w)&=\liminf_{n\to\infty}Y_n+\lim_{n\to\infty}Z_n(w)\\
    &\ge\frac12\|u(T)\|_2^2-\frac12\|u(0)\|_2^2+\lim_{n\to\infty}Z_n(w)\\
    &=\int_0^T \Hinner{\partial_t u}{u} \md t-\int_0^T \Hinner{\partial_t u}{w}\md t - \int_0^T \Hinner{\mG(w)}{u - w}\md t\\
    &=\int_0^T\Hinner{\partial_tu-\mG(w)}{u-w}\md t.
    \end{aligned}
\end{equation}
Besides, by the Lipschitz continuity of $\mG$ in \cref{ass:lipschitz-v2},
\begin{equation}\label{eq:Xn-bound-sm}
\begin{aligned}
    \limsup_{n\to\infty}X_n(w) &\le\limsup_{n\to\infty}\int_0^T \norm{\mG(u_n)-\mG(w)}_2\|u_n - w\|_2\md t\\
    &\le L_3(R)\int_0^T\limsup_{n\to\infty}\|u_n-w\|_2^2\md t=L_3(R)\|u-w\|_{L^2(0,T;\mH)}^2
\end{aligned}
\end{equation}
if $w(t)\in B_\mX(R)$ for all $t$, where we take $R = 1 + C_1$.

We choose $w=u+\lambda\psi$ for a smooth test function $\psi$ and a sufficiently small scalar $\lambda>0$, then
\begin{equation}
    \int_0^T\Hinner{\partial_tu-\mG(u+\lambda\psi)}{\psi}\md t\le\liminf_{n\to\infty}X_n\le\lambda L_3(R) \| \psi \|_{L^2(0,T;\mH)}^2.
\end{equation}
Taking the limit $\lambda\to0^+$, we obtain
\begin{equation}
    \int_0^T\Hinner{\partial_tu-\mG(u)}{\psi}\md t\le0,\quad\forall\psi\in\mX.
\end{equation}
Since the inequality holds for $\pm\psi$, the integral vanishes, which implies that $\partial_tu$ and $\mG(u)$ coincide a.e. on $L^2(0,T;\mX)$.

\hfill\proofbox
\subsection{Proof of \cref{lem:J-norm}}\label{sec:proof-of-J-norm}
Since $\frac12\|\cdot\|_\mX^2$ is Fr\'echet differentiable by \cref{ass:X-constraints}, the limit
\begin{equation}
    J(u)(v)=\frac12\lim_{t\to0}t^{-1}(\|u+tv\|_\mX^2-\|u\|_\mX^2)
\end{equation}
exists for all $v\in\mX$. By the triangle inequality, $\|u+tv\|_\mX\le\|u\|_\mX+|t|\|v\|_\mX$, so
\begin{equation}
    |J(u)(v)|\le\frac12\limsup_{t\to0}|t|^{-1}\left|(\|u\|_\mX+|t|\|v\|_\mX)^2-\|u\|_\mX^2\right|=\|u\|_\mX\|v\|_\mX,
\end{equation}
which gives $\|J(u)\|_{\mX^*}\le\|u\|_\mX$. Meanwhile,
\begin{equation}
    J(u)(u)=\frac12\lim_{t\to0}\frac{\|u+tu\|_\mX^2-\|u\|_\mX^2}t=\frac12\lim_{t\to0}\frac{(1+t)^2-1}{t}\|u\|_\mX^2=\|u\|_\mX^2,
\end{equation}
and therefore $\|J(u)\|_{\mX^*}=\|u\|_\mX$.
\hfill\proofbox

\subsection{Proof of \cref{thm:convergence-zero-dissipation}}\label{sec:proof-convergence-zero-dissipation}
According to \cref{prop:energy-dissipation}, any trajectory $\{u(t)\}_{t\in[0,+\infty)}$ is uniformly bounded in $\mX$. By the Rellich--Kondrachov Theorem, the embedding $\mX=H^1(\Omega) \hookrightarrow L^2(\Omega)$ is compact. Therefore, we have a non-empty $\omega$-limit set $\omega(u(0))$.
    
We show that $\omega(u(0))$ is positively invariant under the flow $S(t)$, and the energy $V$ is a constant on the $\omega$-limit set. Consider any sequence $t_n \to \infty$ such that $u(t_n) \to v\in\omega(u(0))$. For any $\tau>0$,
\begin{equation}
    S(\tau)v = S(\tau)\left( \lim_{n \to \infty} u(t_n) \right) = \lim_{n \to \infty} S(\tau)u(t_n) = \lim_{n \to \infty} u(t_n + \tau),
\end{equation}
which implies $S(\tau)v \in \omega(u_0)$.
If $V$ is not a constant on the $\omega$-limit set, we may assume that $u(t_{n_k})\to v_1$ and $u(t_{m_k})\to v_2$ in $\mH$ as $k\to+\infty$ and $V(v_1)<V(v_2)$. By the definition of convergence, there exists a $K$ large enough such that 
\begin{equation}
    V(u(t_{n_K}))<\frac12V(v_1)+\frac12V(v_2)<V(v_2).
\end{equation}
However, since $V$ is non-increasing on the trajectory, the sequence $V(u(t_{m_k}))$ will eventually be smaller than $V(u(t_{n_K}))$, and thus the corresponding limit $V(v_2)$ is not greater than $V(u(t_{n_K}))$, which is a contradiction.

Consider a trajectory lying on the $\omega$-limit set. By taking the time derivative of $V$ it is easy to verify that the free energy dissipation vanishes, and consequently we have $\omega(u(0))\subseteq\mathcal Z$.
% \end{proof}

% \begin{Corl}[Convergence to equilibrium]\label[Corl]{corl:convergence-equilibrium}
%     If the operator $\mM$ is strictly positive in the sense that there exists $c > 0$ such that
%     \begin{equation}
%         \Xcouple{\xi}{\mM(u)\xi} \geq c\norm{\xi}_{\mX^*}^2, \qquad \forall \xi \in \mX^*,
%     \end{equation}
%     then $\mathcal Z = \mE$, and any trajectory converges asymptotically to the set of static equilibria.
% \end{Corl}

% \begin{proof}
Furthermore, if the elliptic condition \cref{eq:strict-positivity} holds, then for any $\varphi \in \mathcal Z$, we have that 
\begin{equation}
    \Xcouple{\mu(\varphi)}{\mM(\varphi)\mu(\varphi)} = 0
\end{equation}
implies $\mu(\varphi) = 0$, so $\varphi \in \mE$.

\hfill\proofbox

\subsection{Proof of \cref{lem:spectral-of-derivative}}\label{sec:proof-of-spectral-of-derivative}

\hfill\proofbox

\subsection{Proof of \cref{thm:factorization-uniqueness}}\label{sec:proof-factorization-uniqueness}
\begin{Lem}\label[Lem]{lem:spectral-of-derivative}
    Consider a Fr\'echet differentiable functional $V:\mX\to\mR$. The Fourier transform of the Fr\'echet derivative $\mu=\delta V/\delta u$ can be written as
    \begin{equation}
        \hat\mu^{[k]}=\partial_{-k}v=\overline{\partial_k v}=\overline{\hat\mu^{[-k]}},
    \end{equation}
    where we use the notation $\partial_k v={\partial v}/{\partial\hat u^{[k]}}$ for any $k\in\mZ$.
\end{Lem}
\begin{proof}
    By definition we have
\begin{equation}
    \hat\mu^{[k]}=\Xcouple{\frac{\delta V}{\delta u}}{\overline{e_k}}=\left.\frac{\md}{\md\tau}V(u+\tau\overline{e_k})\right|_{\tau=0}=\left.\frac{\md}{\md\tau}v(\hat u+\tau \widehat{e_{-k}})\right|_{\tau=0}=\frac{\partial v}{\partial \hat u^{[-k]}}=\partial_{-k}v.
\end{equation}
Furthermore,
\begin{equation}
\begin{aligned}
    \partial_kv+\partial_{-k}v&=\hat\mu^{[k]}+\hat\mu^{[-k]}=\Xcouple{\frac{\delta V}{\delta u}}{e_k+e_{-k}}\in\mR,\\
    \mi\partial_kv-\mi\partial_{-k}v&=\mi\hat\mu^{[k]}-\mi\hat\mu^{[-k]}=\Xcouple{\frac{\delta V}{\delta u}}{\mi e_{-k}-\mi e_{k}}\in\mR,
\end{aligned}
\end{equation}
which gives $\hat\mu^{[-k]}=\overline{\hat\mu^{[k]}}$ and $\partial_{-k}v=\overline{\partial_kv}$ for all $k$.
\end{proof}
\begin{Lem}\label[Lem]{lem:uniqueness}
    Let $(\mM_1,\mW_1)$ and $(\mM_2,\mW_2)$ be independent of $u$ and satisfy 
    \begin{equation}\label{eq:M1-W1=M2-W2}
        (\mM_1+\mW_1)\mu_1(u)=(\mM_2+\mW_2)\mu_2(u),\quad \mu_1=\frac{\delta V_1}{\delta u},\quad\mu_2=\frac{\delta V_2}{\delta u}
    \end{equation}
    with \cref{ass:symmetry,ass:convolution}. Assume that $V_2$ is irreducible on the spectral space. Define
    \begin{equation}
        Z_1=\left\{k\in\mZ\mid\widehat{\mM}_1^{[k]}+\widehat{\mW}_1^{[k]}\neq0\right\},\quad\text{then}\quad
        \lambda_k:=\frac{\widehat{\mM}_2^{[k]}+\widehat{\mW}_2^{[k]}}{\widehat{\mM}_1^{[k]}+\widehat{\mW}_1^{[k]}}
    \end{equation}
    is a constant $\lambda$ for $k\in Z_1$. Moreover, either $\hat\mu_2^{[k]}\equiv0$ for all $k\in Z_1$ or $\lambda\in\mR$ holds.
\end{Lem}

\begin{proof}
    By transforming \cref{eq:M1-W1=M2-W2} from the physical space to the spectral space, we have
    \begin{equation}
        \left(\widehat{\mM}_1^{[k]}+\widehat{\mW}_1^{[k]}\right)\hat\mu_1^{[k]}=\left(\widehat{\mM}_2^{[k]}+\widehat{\mW}_2^{[k]}\right)\hat\mu_2^{[k]},\quad\forall k\in\mZ.
    \end{equation}
    By the definition of $\lambda_k$ we have
    \begin{equation}\label{eq:mu1-mu2}
        \hat\mu_1^{[k]}=\lambda_k\hat\mu_2^{[k]},\quad\forall k\in Z_1.
    \end{equation}
    Since that $\hat\mu_i^{[k]}=\overline{{\partial v_i}/{\partial\hat u^{[k]}}}$ for $i\in\{1,2\}$ by \cref{lem:spectral-of-derivative}, it follows that
    \begin{equation}
        \frac{\partial\hat\mu_i^{[k]}}{\partial\hat u^{[l]}}=\frac{\partial\hat\mu_i^{[l]}}{\partial\hat u^{[k]}},\quad i\in\{1,2\}
    \end{equation}
    for any indices $k$ and $l$.
    Then, plugging \cref{eq:mu1-mu2} into the above equation gives
    \begin{equation}
        \frac{\partial\lambda_k}{\partial \hat u^{[l]}}\hat\mu_2^{[k]}+\lambda_k\frac{\partial\hat\mu_2^{[k]}}{\partial \hat u^{[l]}}=\frac{\partial\lambda_l}{\partial \hat u^{[k]}}\hat\mu_2^{[l]}+\lambda_l\frac{\partial\hat\mu_2^{[l]}}{\partial \hat u^{[k]}},
    \end{equation}
    which implies
    \begin{equation}
        (\lambda_k-\lambda_l)\overline{\frac{\partial^2v_2}{\partial\hat u^{[k]}\partial\hat u^{[l]}}}=0
    \end{equation}
    as the derivatives of $\lambda_k$ and $\lambda_l$ vanish.
    It follows immediately by the irreducibility of $V_2$ that $\lambda_k=\lambda\in\mC$ for any $k\in Z_1$.
    Suppose that $\hat\mu_2^{[k_0]}\not\equiv0$ for some $k_0\in Z_1$. It is clear that $-k_0\in Z_1$ by \cref{lem:spectral-of-derivative}, then by substituting $k$ with $\pm k_0$ in \cref{eq:mu1-mu2}, we have
    \begin{equation}
        \lambda_{k_0}\hat\mu_2^{[k_0]}=\hat\mu_1^{[k_0]}=\overline{\hat\mu_1^{[-k_0]}}=\bar\lambda_{-k_0}\overline{\hat\mu_2^{[-k_0]}}=\bar\lambda_{-k_0}\hat\mu_2^{[k_0]},
    \end{equation}
    which indicates that $\lambda=\lambda_{k_0}=\bar\lambda_{-k_0}=\bar\lambda$ as a real number.
\end{proof}

\begin{Prop}\label[Prop]{prop:dissipative-uniqueness}
    With the same assumptions as in \cref{lem:uniqueness}, if $\mM_1$ is positive definite and $V_2$ is not a constant functional, then there exists a constant $\lambda>0$ such that
    \begin{equation}
        \mM_2=\lambda\mM_1,\quad\mW_2=\lambda\mW_1,\quad \mu_1=\lambda\mu_2.
    \end{equation}
\end{Prop}
\begin{proof}
    It follows by the positive definiteness of $\mM_1$ that $Z_1=\mZ$. By \cref{lem:uniqueness}, there exists a real constant $\lambda$ such that
    \begin{equation}
        \widehat{\mM}_2^{[k]}+\widehat{\mW}_2^{[k]}=\lambda\widehat{\mM}_1^{[k]}+\lambda\widehat{\mW}_1^{[k]},\quad\mu_1^{[k]}=\lambda\mu_2^{[k]},\quad\forall k\in\mZ.
    \end{equation}
    Due to the positive (semi-)definiteness of $\mM_1$ and $\mM_2$, the positiveness of $\lambda$ follows, and the proof is completed by separating the real and imaginary parts and taking the inverse Fourier transform.
\end{proof}
\begin{Prop}\label[Prop]{prop:conservative-uniqueness}
    With the same assumptions as in \cref{lem:uniqueness}, if $\mM_2=0$ and $\hat\mu_2^{[k]}\not\equiv0$ for all $k\in\mZ$, then
    \begin{equation}
        \mM_1\mu_1(u)\equiv0,\quad\mW_2=\lambda'\mW_1
    \end{equation}
    for a constant $\lambda'\in\mR$. Furthermore, we have $\mu_1=\lambda'\mu_2$ if $\ker\mW_1=0$.
\end{Prop}
\begin{proof}
    Suppose that $\widehat{\mM}_1^{[k_0]}\neq0$ for some $k_0\in\mZ$. It follows that $k_0\in Z_1$ by definition, then we have
    \begin{equation}
        \widehat{\mM}_2^{[k_0]}+\widehat{\mW}_2^{[k_0]}=\lambda\widehat{\mM}_1^{[k_0]}+\lambda\widehat{\mW}_1^{[k_0]},\quad\mu_1^{[k_0]}=\lambda\mu_2^{[k_0]}
    \end{equation}
    for some $\lambda\in\mR$ according to \cref{lem:uniqueness}. Comparing the real parts on both sides leads to $\lambda=0$, which implies $\hat\mu_1^{[k_0]}\equiv0$. Therefore, it can be concluded that $\widehat{\mM}_1^{[k]}\mu_1^{[k]}\equiv0$ for all $k\in\mZ$. Back to the spectral representation of \cref{eq:M1-W1=M2-W2}, we have
    \begin{equation}\label{eq:W1=W2}
        \widehat{\mW}_1^{[k]}\hat\mu_1^{[k]}=\widehat{\mW}_2^{[k]}\hat\mu_2^{[k]},\quad\forall k\in\mZ.
    \end{equation}
    Analogously, we may define $Z_1'=\left\{k\in\mZ\mid\widehat{\mW}_1^{[k]}\neq0\right\}$
    and let $\lambda_k':={\widehat{\mW}_2^{[k]}}/{\widehat{\mW}_1^{[k]}}\in\mR$
    for any $k\in Z_1'$. With the same deduction in \cref{lem:uniqueness}, we can show that $\lambda_k'$ remains a constant for $k\in Z_1'$ because of the irreducibility of $V_2$. Hence $\hat\mu_1^{[k]}=\lambda'\hat\mu_2^{[k]}$ with a constant $\lambda'\in\mR$ for any $k\in Z_1'$. Meanwhile, for $k\notin Z_1'$, $\widehat\mW_1^{[k]}=0$ implies $\widehat\mW_2^{[k]}=0$ according to \cref{eq:W1=W2}. As a result, $\widehat\mW_2^{[k]}=\lambda'\widehat\mW_1^{[k]}$ for any $k$ and thus $\mW_2=\lambda'\mW_1$. Additionally, $\ker\mW_1=0$ implies $Z_1'=\mZ$ and thus $\mu_1=\lambda'\mu_2$.
\end{proof}

\subsection{Proof of \cref{thm:discrete-stability}}\label{sec:proof-discrete-energy-dissipation}
By the Lipschitz continuity of $\mu_1$, applying the Descent Lemma to the step from $u^s$ to $u^{s+1}$, we have that
\begin{equation}\label{eq:Descent-Lemma-sm}
    \begin{aligned}
        V(u^{s+1}) - V(u^s) &=\int_0^1\left\langle \mu(u^s+\tau(u^{s+1}-u^s)),u^{s+1}-u^s\right\rangle_{\mX^*,\mX}\md \tau\\
        &=\int_0^1\left\langle \mu(u^s+\tau(u^{s+1}-u^s))-\mu(u^s),u^{s+1}-u^s\right\rangle_{\mX^*,\mX}\md \tau\\
        &\hspace{15ex}+\langle\mu(u^s),u^{s+1}-u^{s}\rangle_{\mX^*,\mX}\\
        &\le\langle\mu(u^s),u^{s+1}-u^{s}\rangle_{\mX^*,\mX}+L_1(R)\int_0^1\tau\|u^{s+1}-u^s\|_\mX^2\md\tau\\
        &\le \Xcouple{\mu^s}{u^{s+1}-u^s} + \frac{L_1(R)}{2} \norm{u^{s+1}-u^s}_\mX^2
    \end{aligned}
\end{equation}
as long as
\begin{equation}\label{eq:R-range-sm}
    \max\left(\norm{u^s}_\mX,\norm{u^{s+1}}_\mX\right)\le R.
\end{equation}
Combining with the discrete update \cref{eq:discrete-update}, we may conclude that it suffices to ensure that
\begin{equation}\label{eq:L1R-sufficient}
     \left\|[\mM(u^s) + \mW(u^s)] \mu(u^s)\right\|_\mX^2L_1(R)\Delta t\le 2\left\langle \mu(u^s), \mM(u^s) \mu(u^s) \right\rangle_{\mX^*,\mX},
\end{equation}
for some $R$ satisfying \cref{eq:R-range-sm}. Note that in \cref{eq:R-range-sm}, $R$ is partially determined by $\|u_{n+1}\|_\mX$, and we have to dismiss the dependency. 

Now, we claim that either $\mu^s=0$, or
\begin{equation}\label{eq:claim-sm}
    \Delta t\le\frac{2\Xcouple{\mu^s}{\mM^s\mu^s}}{\norm{(\mM^s+\mW^s)\mu^s}_\mX^2L_1(R^s)},\quad R^s=D+\frac{2\Xcouple{\mu^s}{\mM^s\mu^s}}{\norm{(\mM^s+\mW^s)\mu^s}_\mX L_1(D)}
\end{equation}
for the $s$th step will ensure that $V(u^{s+1})\le V(u^s)$ and $\norm{u^s}_\mX\le D$ for all $s$, and we will prove it by induction. The constant $D$ is chosen such that
\begin{equation}
    V(w)\le V(u^0)\implies\|w\|_\mX\le D
\end{equation}
for any $w\in\mX$, and the existence is guaranteed by the coercivity of $V$.
When $\mu^s=0$, the claim is trivial. Otherwise, by triangle inequality, we can decompose the update as
\begin{equation}
\begin{aligned}
    \norm{u^{s+1}}_\mX&\le\norm{u^s}_\mX+\Delta t\norm{(\mM^s+\mW^s)\mu^s}_\mX\\
    &\le\norm{u^s}_\mX+\frac{2\Xcouple{\mu^s}{\mM^s\mu^s}}{\norm{(\mM^s+\mW^s)\mu^s}_\mX L_1(R^s)}\\
    &\le D+\frac{2\Xcouple{\mu^s}{\mM^s\mu^s}}{\norm{(\mM^s+\mW^s)\mu^s}_\mX L_1(D)}=R^s,
\end{aligned}
\end{equation}
which means that $R=R^s$ makes \cref{eq:Descent-Lemma-sm} hold. It follows that
\begin{equation}
\begin{aligned}
    V(u^{s+1})-V(u^s)&\le\Xcouple{\mu^s}{u^{s+1}-u^s} + \frac{L_1(R^s)}{2} \norm{u^{s+1}-u^s}_\mX^2\\
    &=-\Delta t\Xcouple{\mu^s}{\mM^s\mu^s} + \frac{L_1(R^s)}{2} \norm{(\mM^s + \mW^s) \mu^s}_\mX^2(\Delta t)^2\\
    &\le0.
\end{aligned}
\end{equation}
Hence we obtain the non-increasing behavior of $V(u^s)$, and by induction we have $V(u^{s+1})\le V(u^s)\le\cdots\le V(u^0)$, which implies that $\norm{u^{s+1}}_\mX\le D$ as well.

After proving the claim, by \cref{eq:claim-sm} we have
\begin{equation}
\begin{aligned}
    R^s&\le D+\frac{2\Xcouple{\mu^s}{(\mM^s+\mW^s)\mu^s}}{\norm{(\mM^s+\mW^s)\mu^s}_\mX L_1(D)}\le D+\frac{2\norm{\mu^s}_\mX}{L_1(D)}\\
    &\le D+\frac{2(\|\mu(0)\|_\mX+L_1(\norm{u^s}_\mX)\norm{u^s}_\mX}{L_1(D)}\\
    &\le3D+2\norm{\mu(0)}_\mX L_1(D)^{-1}=:R_0.
\end{aligned}
\end{equation}
Therefore, the upper bound for the time step
\begin{equation}
    \frac{2\Xcouple{\mu^s}{\mM^s\mu^s}}{\norm{(\mM^s+\mW^s)\mu^s}_\mX^2L_1(R^s)}=2\kappa L_1(R^s)^{-1}\ge2\kappa L_1(R_0)^{-1}
\end{equation}
Moreover, by \Cref{ass:lipschitz-v2} we have
\begin{equation}
\begin{aligned}
    \kappa=\frac{\Xcouple{\mu^s}{\mM^s\mu^s}}{\norm{(\mM^s+\mW^s)\mu^s}_\mX^2}&\ge\frac{c}{\norm{\mM^s+\mW^s}^2}\\
    &\ge\frac{c}{\left[\norm{\mM(0)+\mW(0)}+L_2(\norm{u^s}_\mX)\norm{u^s}_\mX\right]^2}\\
    &\ge\frac{c}{\left[\norm{\mM(0)+\mW(0)}+L_2(D)D\right]^2}=:\kappa_0
\end{aligned}
\end{equation}
is always positive for elliptic dynamics when \cref{eq:strict-positivity} holds.

For conservative systems where $\mM=0$, a similar deduction as \cref{eq:Descent-Lemma-sm} gives
\begin{equation}
    \left|V(u^{s+1})-V(u^s)\right|\le\frac12 L_1(R)\norm{u^{s+1}-u^s}_\mX^2\le\frac12 L_1(R)(\Delta t)^2\norm{\mW^s\mu^s}_\mX^2.
\end{equation}
whenever \cref{eq:R-range-sm} holds.
By assuming that the discrete trajectory does not blow up in finite time $T$, we set $\tilde R$ as the maximal $\mX$-norm on the trajectory, and the term $\norm{\mW^s\mu^s}_\mX$ can be bounded by $\tilde R$ as well due to their Lipschitz continuity given by \cref{ass:lipschitz-v2}. The boundness of the potential changes across the whole trajectory follows immediately by decomposing the increment into multiple steps.

\hfill\proofbox

\section{PDE models}\label{sec:PDE-models}
We describe the detailed setup for the PDE models involved in this work, including the Allen--Cahn and KdV dynamics.

\subsection{Allen--Cahn dataset}
We consider the following Allen--Cahn dynamics
\begin{equation}
    \partial_tu=\Delta u+\varepsilon^{-2}(u-u^3)=-\big(\underbrace{I}_{\mathcal{M}(u)} + \underbrace{0}_{\mathcal{W}(u)} \big) \Big[ \underbrace{- \Delta u - \varepsilon^{-2}(u-u^3)}_{\frac{\delta V}{\delta u}} \Big]
\end{equation}
with $\varepsilon=0.1$.
The free energy functional $V$ can be defined as
\begin{equation}
    V_\mathrm{AC}=\int_\Omega\left(\frac12|\nabla u|^2-\frac{2u^2-u^4}{4\varepsilon^2}\right)\md x.
\end{equation}

\subsection{KdV dataset}
We consider the following KdV dynamics
\begin{equation}
    \partial_tu=-u_{xxx}-6uu_x=-\big(\underbrace{0}_{\mathcal{M}(u)} + \underbrace{(-\partial_x)}_{\mathcal{W}(u)} \big) \Big[ \underbrace{-u_{xx}-3u^2}_{\frac{\delta V}{\delta u}} \Big].
\end{equation}
The free energy functional $V$ can be defined as
\begin{equation}
    V_\mathrm{KdV}=\int_\Omega\left(\frac12u_x^2-u^3\right)\md x.
\end{equation}

\subsection{Numerical steppers}\label{sec:numerical-stepper}
Consider solving PDE on the spectral domain
\begin{equation}
    \partial_t\hat u_t=-\left[\hat\mM(\hat u_t)+\hat\mW(\hat u_t)\right]\hat\mu_t.
\end{equation}
Suppose that the term $\hat\mu_t$ can be decomposed as a linear function plus a nonlinear function of $\hat u_t$, formulated as
\begin{equation}
    \hat\mu_t=\hat L\hat u_t+\hat N(\hat u_t),
\end{equation}
then we may use the time discretization
\begin{equation}
    \hat u_{t+\Delta t}-\hat u_t=-\left[\hat\mM(\hat u_t)+\hat\mW(\hat u_t)\right]\left[\hat L\hat u_{t+\Delta t}+\hat N(\hat u_t)\right]\Delta t,
\end{equation}
which is equivalent to
\begin{equation}
    \hat u_{t+\Delta t}=\frac{\hat u_t-\left[\hat\mM(\hat u_t)+\hat\mW(\hat u_t)\right]\hat N(\hat u_t)\Delta t}{1+\left[\hat\mM(\hat u_t)+\hat\mW(\hat u_t)\right]\hat L\Delta t}.
\end{equation}
Such an update is usually referred to as the semi-implicit backward differentiation stepper of order 1 (SBDF1).

Analogous to the leap-frog scheme, we may also use the time discretization
\begin{equation}
    \hat u_{t+\Delta t}-\hat u_{t-\Delta t}=-2\left[\hat\mM(\hat u_t)+\hat\mW(\hat u_t)\right]\left[\hat L\hat u_{t+\Delta t}+\hat N(\hat u_t)\right]\Delta t,
\end{equation}
and the resulted scheme is abbreviated as SBDF2.

\section{Microscopic chain models}
Consider a one-dimensional chain where the displacement of the $n$th node is denoted by $q_n$. The strain between the $n$th and the $(n+1)$th nodes is defined as
\begin{equation}
  r_n = q_{n+1} - q_n.
\end{equation}

\subsection{FPUT chain model}
The Fermi--Pasta--Ulam--Tsingou (FPUT) chain model employs a truncated quadratic force law given by
\begin{equation}
  F(r) = c^2 r + \alpha r^2.
\end{equation}
Consequently, the equations of motion for the node displacements can be written as
\begin{equation}
  \ddot q_n = F(r_n) - F(r_{n-1}).
\end{equation}
This system can be cast in Hamiltonian form, with the Hamiltonian defined as
\begin{equation}
  H = \sum_n \left(\frac{1}{2}v_n^2 + V(r_n)\right),
  \qquad F(r) = V'(r),
\end{equation}
yielding the canonical equations
\begin{equation}
\begin{cases}
  \dot q_n = \dfrac{\partial H}{\partial v_n} = v_n,\\[1em]
  \dot v_n = -\dfrac{\partial H}{\partial q_n} = F(r_n) - F(r_{n-1}).
\end{cases}
\end{equation}
Taking the second time derivative of the strain yields the governing equation for $r_n$
\begin{equation}\label{eq:strain}
  \ddot r_n = \ddot q_{n+1} - \ddot q_n = c^2(r_{n+1} - 2r_n + r_{n-1}) + \alpha(r_{n+1}^2 - 2r_n^2 + r_{n-1}^2).
\end{equation}

To capture the continuum limit, we adopt a long-wave, small-amplitude ansatz describing a right-moving wave
\begin{equation}
  r_n(t) = \varepsilon^2 u(\xi,\tau),
  \qquad
  \xi = \varepsilon(n - ct),
  \qquad
  \tau = \varepsilon^p t,
\end{equation}
where the temporal scaling exponent $p>1$ will be determined via dominant balance. Applying the chain rule, the differential time operator becomes
\begin{equation}
  \partial_t = -c\varepsilon \partial_\xi + \varepsilon^p \partial_\tau.
\end{equation}
Consequently, the time derivatives of the strain are
\begin{equation}
    \dot r_n = -c\varepsilon^3 u_\xi + \varepsilon^{p+2} u_\tau
\end{equation}
and
\begin{equation}\label{eq:compare1}
  \ddot r_n = \varepsilon^2 \left(c^2\varepsilon^2 u_{\xi\xi} - 2c\varepsilon^{p+1} u_{\xi\tau} + \varepsilon^{2p} u_{\tau\tau}\right)
  = c^2 \varepsilon^4 u_{\xi\xi} - 2c \varepsilon^{p+3} u_{\xi\tau} + \mathcal O(\varepsilon^{2p+2}).
\end{equation}
Expanding the spatial shift terms $r_{n\pm 1} = \varepsilon^2 u(\xi \pm \varepsilon,\tau)$ in $\varepsilon$ using a Taylor series gives
\begin{equation}
  r_{n+1} - 2r_n + r_{n-1} = \varepsilon^4 u_{\xi\xi} + \frac{\varepsilon^6}{12} u_{\xi\xi\xi\xi} + \mathcal O(\varepsilon^8),
\end{equation}
and
\begin{equation}
  r_{n+1}^2 - 2r_n^2 + r_{n-1}^2 = \varepsilon^6 (u^2)_{\xi\xi} + \mathcal O(\varepsilon^8).
\end{equation}
Substituting these expansions back into \cref{eq:strain} yields
\begin{equation}\label{eq:compare2}
  \ddot r_n = c^2 \varepsilon^4 u_{\xi\xi} + \varepsilon^6 \left( \frac{c^2}{12} u_{\xi\xi\xi\xi} + \alpha (u^2)_{\xi\xi} \right) + \mathcal O(\varepsilon^8).
\end{equation}
Equating the terms in \cref{eq:compare1} and \cref{eq:compare2}, dominant balance dictates a cubic slow-time scale ($p=3$). Collecting the $\mathcal O(\varepsilon^6)$ terms leaves
\begin{equation}
  -2cu_{\xi\tau} = \frac{c^2}{12} u_{\xi\xi\xi\xi} + \alpha (u^2)_{\xi\xi}.
\end{equation}
Assuming periodic or rapidly decaying boundary conditions, integrating with respect to $\xi$ yields
\begin{equation}
    u_\tau + \frac\alpha{2c}(u^2)_\xi + \frac c{24}u_{\xi\xi\xi} = 0,
\end{equation}
which is the standard form of the Korteweg–de Vries (KdV) equation
\begin{equation}
  u_t + a_{\mathrm{kdv}}uu_x + b_{\mathrm{kdv}}u_{xxx} = 0,
\end{equation}
with $a_\mathrm{kdv}=\alpha/c$ and $b_\mathrm{kdv}=c/24$.

\subsection{FENE chain model}
Alternatively, the Finitely Extensible Nonlinear Elastic (FENE) model utilizes the potential
\begin{equation}
    V(r) = -\frac{HR^2}{2} \log(1 - (r/R)^2),
\end{equation}
which yields the force law
\begin{equation}
    F(r) = \frac{Hr}{1 - (r/R)^2},
\end{equation}
where the parameters $H$ and $R$ are fixed. Because the macroscopic continuum limit of this specific force law does not readily reduce to a known, standard PDE, it provides a nontrivial testbed for evaluating the capabilities of our framework.

\subsection{Numerical simulation}\label{sec:chain-simulation}
In this section, we detail the numerical simulation of the microscopic chain. We illustrate the procedure using the chain-KdV model, and an analogous approach applies to the FENE model.

Recall that the microscopic chain operates in laboratory time $t$, whereas the macroscopic KdV equation evolves in slow time $\tau=\varepsilon^3t$ for a fixed scaling factor $0 < \varepsilon \ll 1$. Given an initial macroscopic profile $u(\xi,0)$, the corresponding initial conditions for the microscopic strain and strain rate are
\begin{equation}
  r_n(0) = \varepsilon^2 u(\varepsilon n,0),
  \qquad
  \dot r_n(0) = -c\varepsilon^3 u_\xi(\varepsilon n,0).
\end{equation}

We integrate the system in time using the velocity Verlet algorithm. Introducing a micro time-step $\delta t$ such that the macroscopic saving interval is $\Delta t = m \delta t$ for $m$ substeps, the update rules for a single substep are as follows.
\begin{enumerate}
    \item Compute the half-step velocity $\dot r_{n}^{1/2} = \dot r_n + \frac{1}{2}\ddot r_n\delta t$, where $\ddot r_n$ is evaluated via \cref{eq:strain} using the current strain $r_n$.
    \item Update the strain $r_{n}^{+} = r_n + \dot r_{n}^{1/2}\delta t$.
    \item Compute the full-step velocity $\dot r_{n}^{+} = \dot r_{n}^{1/2} + \frac{1}{2}\ddot r_{n}^{+}\delta t$, where $\ddot r_{n}^{+}$ is evaluated via \cref{eq:strain} using the updated strain $r_{n}^{+}$.
\end{enumerate}
After advancing the system by $m$ substeps, we reconstruct the macroscopic state $u$ at time $\tau + \Delta\tau$ (where $\Delta\tau = \varepsilon^3\Delta t$) via the inverse relation
\begin{equation}
    u(\varepsilon(n-ct), \varepsilon^3 t) = \varepsilon^{-2}r_n(t)
\end{equation}
for each nodal index $n$. In practice, we enforce periodic boundary conditions by applying a modulo operation to the spatial coordinates relative to the domain length, and subsequently interpolate the scattered $u$ values onto a uniform spatial grid.

\section{Data generation}\label{sec:data-generation}
% Our experiments involve four test cases, including two PDE cases (KdV and Allen--Cahn equations) and two microscopic cases (FPU and FENE chains).

% To generate the reference trajectories for the PDE datasets, we adopt the Python package \texttt{dedalus}\footnote{\url{https://github.com/dedalusProject/dedalus}} and apply the SBDF2 time stepper on the spectral space. We set the step size as $0.001$, with 25 inner integral substeps. the simulation horizon is $T=0.1$, which gives 100 snapshots for a single trajectory.

% For the FPU chain model, we select parameters corresponding to a canonical KdV equation with $a_\mathrm{kdv}=1$ and $b_\mathrm{kdv}=1/24$, which dictates $c=1$ and $\alpha=1$. For the FENE chain model, we set $H=1$ and $R=50\varepsilon^2$. The scaling parameter $\varepsilon$ is fixed at $0.05$ and $0.03$ for the chain--KdV and FENE models, respectively.

% All four datasets share the same underlying idea: a random superposition of some sine waves for the initial conditions. For the PDE datasets, we direct initialize the solution variable $u$, while for the chain models, we initialize the macripscopic state $u$ and then convert it to the strain variable $r$.

Our numerical experiments encompass four distinct test cases: two continuum PDE models (the KdV and Allen--Cahn equations) and two microscopic particle systems (the FPUT and FENE chain models).

% To generate high-fidelity reference trajectories for the PDE datasets, we employ the \texttt{dedalus}\footnote{\url{https://github.com/dedalusProject/dedalus}} Python framework, utilizing a second-order semi-implicit backward differentiation formula (SBDF2) in the spectral domain. 
We set the recording step size to $\Delta t = 0.001$, resolving the internal dynamics with $25$ intermediate integration substeps per recorded step. The total simulation horizon is $T=0.1$, yielding exactly $100$ temporally equidistant snapshots per trajectory. For 1D cases, we set the grid size as 256, and for 2D cases, we set the spatial resolution as $128\times128$. The data generation process for the PDE models are implemented via the Python package \texttt{dedalus}\footnote{\url{https://dedalus-project.org/}} \cite{dedalus} with the SBDF2 numerical stepper. The generation for the microscopic chain models follows \cref{sec:chain-simulation}.

For the microscopic systems, the FPUT chain parameters are selected to correspond to a canonical continuum KdV equation with $a_\mathrm{kdv}=1$ and $b_\mathrm{kdv}=1/24$, which fixes the microscopic parameters at $c=1$ and $\alpha=1$. For the FENE chain model, we set the parameters to $H=1$ and $R=50\varepsilon^2$. The spatial scaling parameter $\varepsilon$ is fixed at $0.05$ for the FPUT chain and $0.03$ for the FENE chain.

The initial conditions across all four datasets are generated using a consistent methodology based on random superpositions of sinusoidal waves. For the PDE datasets, these superpositions are used to directly initialize the continuum solution variable $u$. For the particle chain models, we first generate a continuum state $u$, which is subsequently projected onto the discrete microscopic strain variables $r$.
\section{Networks and training details}\label{sec:network-and-training}
\subsection{Network architectures}
We describe the detailed architecture for each model below.
\paragraph{FNO} We employ the FNO module officially implemented in the Python package 
\texttt{neuralop}\footnote{\url{https://github.com/neuraloperator/neuraloperator}} \cite{kossaifi2025neuralop}. We set the number of modes as 24, and the number of hidden channels as 32. All the other parameters remain unchanged.
\paragraph{OnsagerNet} We use the classical OnsagerNet formulation that treats the $N_x$-point PDE grid directly as an $N_x$-dimensional ODE, evolving under
\begin{equation}
    \ddt{\mathbf{u}} = -\bigl(M(\mathbf{u}) + W(\mathbf{u})\bigr)\nabla V(\mathbf{u}),
\end{equation}
where $V$ is a coercive MLP potential with 3 hidden layers of width 256 and \texttt{tanh} activations; $M(\mathbf{u}) = \epsilon \mathbf{I} + B(\mathbf{u})B(\mathbf{u})^\top$ is a low-r   ank SPD dissipation matrix of rank 8 with $\epsilon = 10^{-3}$, and $B$ parameterized by a 2-hidden-layer MLP of width 256; and $W(\mathbf{u}) = B(\mathbf{u})C^\top - CB(\mathbf{u})^\top$ is a low-rank anti-symmetric conservation matrix of rank 4, similarly parameterized.
\paragraph{\ournetabbrv} We use a pointwise MLP applied independently at each grid point to implement $F_\phi$. For the gradient term $\nabla v_\phi$, we represent $v_\phi$ as a complex-valued MLP acting on the first 16 Fourier modes with 3 hidden layers of width 64, and the graidient is computed via automatic differentiation. All the activations in $F_\phi$ and $v_\phi$ are set as SiLU \cite{hendrycks2016GELUs}, except for the activations for the final nonlinear layer where we set them as sine to ensure the boundness of the network output.
\subsection{Training details}\label{sec:training-details}
% The time step sizes are set the same as the reference, and we fix the number of gradient accumulation steps as $K=5$ for the chain-KdV case and $K=3$ for the other cases.

% \begin{table}[h]
%     \centering
%     \begin{tabular}{l|c}
%     \toprule
%     Model & number of parameters \\
%     \midrule
%     Classical solver &  --- \\
%     S-OnsagerNet-V & 258 \\
%     \midrule
%     FNO\cite{li2020fourier} & 66.3K \\
%     OnsagerNet\cite{Yu2021OnsagerNet}& 1.3M \\
%     Res-OnsagerNet & 790K \\
%     S-OnsagerNet & 18.2K \\
%     \bottomrule
%     \end{tabular}
%     \caption{5-step relative prediction error ($\downarrow$) for various baselines.}
%     \label{tab:benchmark}
% \end{table}

We adopt the Adam optimizer with learning rate $10^{-3}$. No weight decay is used. To manage the learning rate, we decrease the learning rate to half if the validation loss does not decrease for 300 epochs. We validate the networks every 100 epochs, and terminate the training process when the validation loss does not decrease for 5 validation steps. For the loss function, we fix the number of gradient accumulation steps as $K=5$ for the FPUT model and $K=3$ for the other cases.
\section{Ablation study}
We have conducted an ablation study for the modeling choice of our \ournet. Apart from the baselines (classical solver and FNO), we introduce the following variants of our model.
\paragraph{Residual OnsagerNet (abbreviated as Res-OnsagerNet)}
Instead of using \cref{eq:parameterization-mu}, we introduce
\begin{equation}
    \hat\mu_\theta^\mathrm{Res}(u^s)=\left[\alpha+(2\pi \mathbf{k})^2\beta\right]\odot\hat u^s+H_\phi(\hat u^s).
\end{equation}
where we directly parameterize all the gradient terms $H_\phi$ as a complex-valued network. This is the most naive approach to learning the increment in the spectral space. Note that in such circumstances, the parameterization for the term $\mu_t$ is not guaranteed to be a functional derivative.
\paragraph{{\ournet} with real $V$ (abbreviated as \ournetabbrv-V)}
In the update \cref{eq:parameterization-mu} of \ournetabbrv, we set the potential $V$ as the real potential for the PDE models. More specifically, we set $\hat\mu_\theta(u^s)$ as $V_\mathrm{KdV}$ and $V_{\mathrm{AC}}$ for the KdV equation and the Allen--Cahn equation, respectively.
\begin{table}[h]
    \centering
    \resizebox{\textwidth}{!}{%
    \begin{tabular}{l|c|c|c|c|c}
    \toprule
    Model & \#params& KdV & Allen--Cahn & Chain--KdV & FENE \\
    \midrule
    Classical solver & --- & $0.0541^{\pm0.0313}$ & ${0.0085}^{\pm0.0029}$ & --- & --- \\
    \ournetabbrv-V & 258 & $0.0096^{\pm0.0039}$ & $0.0059^{\pm0.0020}$ & --- & --- \\
    \midrule
    FNO\cite{li2020fourier} & 66.3 K & $0.1747^{\pm0.0820}$ & $0.0204^{\pm0.0081}$ & $0.0360^{\pm0.0096}$ & $0.0186^{\pm0.0066}$ \\
    OnsagerNet\cite{Yu2021OnsagerNet}& 1.3 M & $0.0671^{\pm0.0241}$ & $0.0260^{\pm0.0037}$ & $0.0242^{\pm0.0075}$ & $0.0139^{\pm0.0076}$ \\
    Res-OnsagerNet & 790 K & $0.0129^{\pm0.0027}$ & $0.1225^{\pm0.0277}$ & $0.0156^{\pm0.0043}$ & $0.0100^{\pm0.0058}$ \\
    \ournetabbrv & \textbf{18.2 K} & $\mathbf{0.0100}^{\pm0.0037}$ & $\mathbf{0.0055}^{\pm0.0022}$ & $\mathbf{0.0082}^{\pm0.0026}$ & $\mathbf{0.0012}^{\pm0.0004}$ \\
    \bottomrule
    \end{tabular}%
    }
    \caption{5-step relative prediction error ($\downarrow$) for various baselines. Some of the entries have been exhibited in \cref{tab:benchmark} from the main text.}
    \label{tab:ablation}
\end{table}

\bibliographystyle{siamplain}
\bibliography{refs}
\end{document}